\documentclass{article}

\usepackage[preprint,nonatbib]{nips_2018}

\usepackage[utf8]{inputenc} 
\usepackage[T1]{fontenc}    
\usepackage{url}            
\usepackage{booktabs}       
\usepackage{amsfonts}       
\usepackage{nicefrac}       
\usepackage{microtype}      
\usepackage{amsmath}
\usepackage{algorithm}
\usepackage{algorithmic}
\usepackage{graphicx}
\usepackage{subfigure}

\newtheorem{lemma}{Lemma}
\newtheorem{theorem}{Theorem}
\newtheorem{corollary}{Corollary}
\newtheorem{definition}{Definition}
\newtheorem{remark}{Remark}
\newenvironment{proof}{{\noindent\it Proof}\quad}{\hfill $\square$\par}

\title{Proximal SCOPE for distributed sparse learning: Better data partition implies faster convergence rate}

\author{
 Shen-Yi Zhao \\
  National Key Lab. for Novel Software Tech. \\
Dept. of Comp. Sci. and Tech. \\
Nanjing University, Nanjing 210023, China \\
  \texttt{zhaosy@lamda.nju.edu.cn} \\
  \And
  Gong-Duo Zhang \\
  National Key Lab. for Novel Software Tech. \\
Dept. of Comp. Sci. and Tech. \\
Nanjing University, Nanjing 210023, China \\
  \texttt{zhanggd@lamda.nju.edu.cn} \\
  \And
  Ming-Wei Li \\
  National Key Lab. for Novel Software Tech. \\
Dept. of Comp. Sci. and Tech. \\
Nanjing University, Nanjing 210023, China \\
  \texttt{limw@lamda.nju.edu.cn} \\
  \And
  Wu-Jun Li \\
  National Key Lab. for Novel Software Tech. \\
Dept. of Comp. Sci. and Tech. \\
Nanjing University, Nanjing 210023, China \\
  \texttt{liwujun@nju.edu.cn} \\
}

\def\a{{\bf a}}
\def\b{{\bf b}}
\def\c{{\bf c}}

\def\g{{\bf g}}

\def\r{{\bf r}}

\def\u{{\bf u}}
\def\v{{\bf v}}
\def\w{{\bf w}}
\def\x{{\bf x}}
\def\y{{\bf y}}
\def\z{{\bf z}}

\def\A{{\bf A}}

\def\E{{\bf E}}

\def\P{{\bf P}}

\def\0{{\bf 0}}
\def\1{{\bf 1}}
\def\2{{\bf 2}}
\def\3{{\bf 3}}
\def\4{{\bf 4}}
\def\5{{\bf 5}}
\def\6{{\bf 6}}
\def\7{{\bf 7}}
\def\8{{\bf 8}}
\def\9{{\bf 9}}

\def\EB{{\mathbb E}}

\def\PB{{\mathbb P}}

\def\RB{{\mathbb R}}

\begin{document}

\maketitle

\begin{abstract}
  Distributed sparse learning with a cluster of multiple machines has attracted much attention in machine learning, especially for large-scale applications with high-dimensional data. One popular way to implement sparse learning is to use $L_1$ regularization. In this paper, we propose a novel method, called proximal \mbox{SCOPE}~(\mbox{pSCOPE}), for distributed sparse learning with $L_1$ regularization. \mbox{pSCOPE} is based on a \underline{c}ooperative \underline{a}utonomous \underline{l}ocal \underline{l}earning~(\mbox{CALL}) framework. In the \mbox{CALL} framework of \mbox{pSCOPE}, we find that the data partition affects the convergence of the learning procedure, and subsequently we define a metric to measure the goodness of a data partition. Based on the defined metric, we theoretically prove that pSCOPE is convergent with a linear convergence rate if the data partition is good enough. We also prove that better data partition implies faster convergence rate. Furthermore, pSCOPE is also communication efficient. Experimental results on real data sets show that pSCOPE can outperform other state-of-the-art distributed methods for sparse learning.
\end{abstract}

\section{Introduction}

Many machine learning models can be formulated as the following regularized empirical risk minimization problem:
\begin{align}\label{eq:obj}
\mathop{\min}_{\w \in \RB^d}~P(\w) = \frac{1}{n} \sum_{i=1}^{n} f_i(\w) + R(\w),
\end{align}
where $\w$ is the parameter to learn, $f_i(\w)$ is the loss on training instance $i$, $n$ is the number of training instances, and $R(\w)$  is a regularization term. Recently, sparse learning, which tries to learn a sparse model for prediction, has become a hot topic in machine learning. There are different ways to implement sparse learning~\cite{Tibshirani94regressionshrinkage,DBLP:conf/icml/WangKS017}. One popular way is to use $L_1$ regularization, i.e., $R(\w) = \lambda\|\w\|_1$. In this paper, we focus on sparse learning with $R(\w) = \lambda\|\w\|_1$. Hence, in the following content of this paper, $R(\w) = \lambda\|\w\|_1$ unless otherwise stated.

One traditional method to solve~(\ref{eq:obj}) is proximal gradient descent~(pGD)~\cite{DBLP:journals/siamis/BeckT09}, which can be written as follows:
\begin{align}\label{eq:proxgd}
  \w_{t+1} = \mbox{prox}_{R,\eta}(\w_t - \eta\nabla F(\w_t)),
\end{align}
where $F(\w) = \frac{1}{n} \sum_{i=1}^{n} f_i(\w)$, $\w_t$ is the value of $\w$ at iteration $t$, $\eta$ is the learning rate, $prox$ is the proximal mapping defined as
\begin{align}\label{eq:proximal mapping}
  \mbox{prox}_{R,\eta}(\u) = \mathop{\arg\min}_{\v}(R(\v) + \frac{1}{2\eta}\|\v - \u\|^2).
\end{align}

Recently, stochastic learning methods, including stochastic gradient descent~(\mbox{SGD})~\cite{DBLP:journals/siamjo/NemirovskiJLS09}, stochastic average gradient~(SAG)~\cite{DBLP:journals/mp/SchmidtRB17}, stochastic variance reduced gradient~(SVRG)~\cite{DBLP:conf/nips/Johnson013},  and stochastic dual coordinate ascent~(\mbox{SDCA})~\cite{DBLP:journals/jmlr/Shalev-Shwartz013}, have been proposed to speedup the learning procedure in machine learning. Inspired by the success of these stochastic learning methods, proximal stochastic methods, including proximal SGD~(pSGD)~\cite{DBLP:conf/nips/LangfordLZ08,DBLP:journals/jmlr/DuchiS09,DBLP:conf/pkdd/ShiL15,DBLP:journals/siamjo/ByrdHNS16}, proximal block coordinate descent~(pBCD)~\cite{Tseng:2001:CBC:565614.565615,citeulike:9641309,DBLP:conf/icml/ScherrerHTH12}, proximal SVRG~(pSVRG)~\cite{DBLP:journals/siamjo/Xiao014} and proximal SDCA~(pSDCA)~\cite{DBLP:conf/icml/Shalev-Shwartz014}, have also been proposed for sparse learning in recent years. All these proximal stochastic methods are sequential~(serial) and implemented with one single thread.

The serial proximal stochastic methods may not be efficient enough for solving large-scale sparse learning problems. Furthermore, the training set might be distributively stored on a cluster of multiple machines in some applications. Hence, distributed sparse learning~\cite{DBLP:conf/icml/AybatWI15} with a cluster of multiple machines has attracted much attention in recent years, especially for large-scale applications with high-dimensional data. In particular, researchers have recently proposed several distributed proximal stochastic methods for sparse learning~\cite{DBLP:journals/corr/LiXZL16,DBLP:conf/aaai/MengCYWML17,DBLP:journals/jmlr/LeeLMY17,DBLP:journals/jmlr/MahajanKS17,DBLP:journals/corr/SmithFJJ15}\footnote{In this paper, we mainly focus on distributed sparse learning with $L_1$ regularization. The distributed methods for non-sparse learning, like those in~\cite{DBLP:conf/nips/ReddiHSPS15,DBLP:conf/icdm/DeG16,DBLP:conf/aistats/LeblondPL17}, are not considered.}.

One main branch of the distributed proximal stochastic methods includes distributed pSGD~(dpSGD)~\cite{DBLP:journals/corr/LiXZL16}, distributed pSVRG~(dpSVRG)~\cite{DBLP:journals/corr/HuoGH16,DBLP:conf/aaai/MengCYWML17} and distributed SVRG~(DSVRG)~\cite{DBLP:journals/jmlr/LeeLMY17}. Both dpSGD and dpSVRG adopt a centralized framework and mini-batch based strategy for distributed learning. One typical implementation of a centralized framework is based on Parameter Server~\cite{DBLP:conf/osdi/LiAPSAJLSS14,DBLP:conf/kdd/XingHDKWLZXKY15}, which supports both synchronous and asynchronous communication strategies. One shortcoming of dpSGD and dpSVRG is that the communication cost is high. More specifically, the communication cost of each epoch is $O(n)$, where $n$ is the number of training instances. DSVRG adopts a decentralized framework with lower communication cost than dpSGD and dpSVRG. However, in DSVRG only one worker is updating parameters locally and all other workers are idling at the same time.

Another branch of the distributed proximal stochastic methods is based on block coordinate descent~\cite{DBLP:conf/icml/BradleyKBG11,DBLP:journals/mp/RichtarikT16,DBLP:journals/siamrev/FercoqR16,DBLP:journals/jmlr/MahajanKS17}. Although in each iteration these methods update only a block of coordinates, they usually have to pass through the whole data set. Due to the partition of data, it also brings high communication cost in each iteration.

Another branch of the distributed proximal stochastic methods is based on SDCA. One representative is PROXCOCOA$+$~\cite{DBLP:journals/corr/SmithFJJ15}. Although PROXCOCOA$+$ has been theoretically proved to have a linear convergence rate with low communication cost, we find that it is not efficient enough in experiments.


In this paper, we propose a novel method, called proximal \mbox{SCOPE}~(\mbox{pSCOPE}), for distributed sparse learning with $L_1$ regularization. pSCOPE is a proximal generalization of the \underline{s}calable \underline{c}omposite \underline{op}timization for l\underline{e}arning~(\mbox{SCOPE})~\cite{DBLP:conf/aaai/ZhaoXSGL17}. SCOPE cannot be used for sparse learning, while pSCOPE can be used for sparse learning. The contributions of pSCOPE are briefly summarized as follows:
\begin{itemize}
\item pSCOPE is based on a \underline{c}ooperative \underline{a}utonomous \underline{l}ocal \underline{l}earning~(\mbox{CALL}) framework. In the \mbox{CALL} framework, each worker in the cluster performs autonomous local learning based on the data assigned to that worker, and the whole learning task is completed by all workers in a cooperative way. The CALL framework is communication efficient because there is no communication during the inner iterations of each epoch.
\item pSCOPE is theoretically guaranteed to be convergent with a linear convergence rate if the data partition is good enough, and better data partition implies faster convergence rate. Hence, pSCOPE is also computation efficient.
\item In pSCOPE, a \emph{recovery strategy} is proposed to reduce the cost of proximal mapping when handling high dimensional sparse data.
\item Experimental results on real data sets show that \mbox{pSCOPE} can outperform other state-of-the-art distributed methods for sparse learning.
\end{itemize}

\section{Preliminary}

In this paper, we use $\|\cdot\|$ to denote the $L_2$ norm $\|\cdot\|_2$, $\w^*$ to denote the optimal solution of (\ref{eq:obj}). For a vector $\a$, we use $a^{(j)}$ to denote the $j$th coordinate value of $\a$. $[n]$ denotes the set $\{1,2,\ldots,n\}$. For a function $h(\a;\b)$, we use $\nabla h(\a;\b)$ to denote the gradient of $h(\a;\b)$ with respect to~(w.r.t.) the first argument $\a$. Furthermore, we give the following definitions.

\begin{definition}\label{def:smooth}
We call a function $h(\cdot)$ is $L$-smooth if it is differentiable and there exists a positive constant $L$ such that $\forall \a,\b: h(\b) \leq h(\a) + \nabla h(\a)^T(\b-\a) + \frac{L}{2}\| \a-\b \|^2$.
\end{definition}

\begin{definition}\label{def:strongly convex}
We call a function $h(\cdot)$ is convex if there exists a constant $\mu\geq 0$ such that $\forall \a,\b: h(\b) \geq h(\a) + \zeta^T(\b-\a) + \frac{\mu}{2}\| \a-\b \|^2$, where $\zeta \in \partial h(\a) = \{\c| h(\b) \geq h(\a) + \c^T(\b-\a), \forall \a,\b\}$. If $h(\cdot)$ is differentiable, then $ \zeta = \nabla h(\a)$. If $\mu > 0$, $h(\cdot)$ is called $\mu$-strongly convex.
\end{definition}

Throughout this paper, we assume that $R(\w)$ is convex, $F(\w) = \frac{1}{n} \sum_{i=1}^{n} f_i(\w)$ is strongly convex and each $f_i(\w)$ is \mbox{smooth}. We do not assume that each $f_i(\w)$ is convex.

\section{Proximal SCOPE}
\label{sec:Algorithm of Proximal SCOPE}

In this paper, we focus on distributed learning with one master~(server) and $p$ workers in the cluster, although the algorithm and theory of this paper can also be easily extended to cases with multiple servers like the Parameter Server framework~\cite{DBLP:conf/osdi/LiAPSAJLSS14,DBLP:conf/kdd/XingHDKWLZXKY15}.

The parameter $\w$ is stored in the master, and the training set $D=\{\x_i,y_i\}_{i=1}^n$ are partitioned into $p$ parts denoted as $D_1, D_2,\ldots,D_p$. Here, $D_k$ contains a subset of instances from $D$, and $D_k$ will be assigned to the $k$th worker. $D = \bigcup_{k=1}^p D_k$.  Based on this data partition scheme, the proximal \mbox{SCOPE}~(\mbox{pSCOPE}) for distributed sparse learning is presented in Algorithm~\ref{alg:proxsco}. The main task of master is to add and average vectors received from workers. Specifically, it needs to calculate the full gradient $\z = \nabla F(\w_t)=\frac{1}{n}\sum_{k=1}^p \z_k$. Then it needs to calculate $\w_{t+1}=\frac{1}{p} \sum_{k=1}^p \u_{k,M}$. The main task of workers is to update the local parameters $\u_{1,m}, \u_{2,m},\ldots,\u_{p,m}$ initialized with $\u_{k,0} = \w_t$. Specifically, for each worker $k$, after it gets the full gradient $\z$ from master, it calculates a stochastic gradient
\begin{align}\label{eq:stovec}
\v_{k,m} = \nabla f_{i_{k,m}}(\u_{k,m}) - \nabla f_{i_{k,m}}(\w_t) + \z,
\end{align}
and then update its local parameter $\u_{k,m}$ by a proximal mapping with learning rate $\eta$:
\begin{align}\label{eq:update}
\u_{k,m+1} = \mbox{prox}_{R,\eta}(\u_{k,m} - \eta\v_{k,m}).
\end{align}

\begin{algorithm}[t]
\caption{Proximal SCOPE}
\label{alg:proxsco}
\begin{algorithmic}[1]
\STATE Initialize $\w_0$ and the learning rate $\eta$;
\STATE \textbf{Task of master}:
\FOR{$t=0,1,2,...T-1$}
\STATE Send $\w_t$ to each worker;
\STATE Wait until receiving $\z_1, \z_2,\ldots,\z_p$ from all workers;
\STATE Calculate the full gradient $\z = \frac{1}{n} \sum_{k=1}^p \z_k$ and send $\z$ to each worker;
\STATE Wait until receiving $\u_{1,M}, \u_{2,M},\ldots,\u_{p,M}$ from all workers and calculate $\w_{t+1} = \frac{1}{p} \sum_{k=1}^p \u_{k,M}$;
\ENDFOR
\\
\STATE  \textbf{Task of the $k$th worker}:
\FOR{$t=0,1,2,...T-1$}
\STATE Wait until receiving $\w_t$ from master;
\STATE Let $\u_{k,0} = \w_t$, calculate $\z_k = \sum_{i\in D_k} f_i(\w_t)$ and send $\z_k$ to master;
\STATE Wait until receiving $\z$ from master;
\FOR{$m=0,1,2,...M-1$}
\STATE Randomly choose an instance $\x_{i_{k,m}} \in D_k$;
\STATE Calculate $\v_{k,m} = \nabla f_{i_{k,m}}(\u_{k,m}) - \nabla f_{i_{k,m}}(\w_t) + \z$;
\STATE Update $\u_{k,m+1} = \mbox{prox}_{R,\eta}(\u_{k,m} - \eta\v_{k,m})$;
\ENDFOR
\STATE Send $\u_{k,M}$ to master
\ENDFOR
\end{algorithmic}
\end{algorithm}

From Algorithm~\ref{alg:proxsco}, we can find that pSCOPE is based on a \underline{c}ooperative \underline{a}utonomous \underline{l}ocal \underline{l}earning~(\mbox{CALL}) framework. In the \mbox{CALL} framework, each worker in the cluster performs autonomous local learning based on the data assigned to that worker, and the whole learning task is completed by all workers in a cooperative way. The cooperative operation is mainly adding and averaging in the master. During the autonomous local learning procedure in each outer iteration which contains $M$ inner iterations~(see Algorithm~\ref{alg:proxsco}), there is no communication. Hence, the communication cost for each epoch of pSCOPE is constant, which is much less than the mini-batch based strategy with $O(n)$ communication cost for each epoch~\cite{DBLP:journals/corr/LiXZL16,DBLP:journals/corr/HuoGH16,DBLP:conf/aaai/MengCYWML17}.

pSCOPE is a proximal generalization of SCOPE~\cite{DBLP:conf/aaai/ZhaoXSGL17}.  Although pSCOPE is mainly motivated by sparse learning with $L_1$ regularization, the algorithm and theory of pSCOPE can also be used for smooth regularization like $L_2$ regularization. Furthermore, when the data partition is good enough, pSCOPE can avoid the extra term $c(\u_{k,m} - \w_t)$ in the update rule of SCOPE, which is necessary for convergence guarantee of SCOPE.

\section{Effect of Data Partition}
\label{sec:FDP}

In our experiment, we find that the data partition affects the convergence of the learning procedure. Hence, in this section we propose a metric to measure the goodness of a data partition, based on which the convergence of pSCOPE can be theoretically proved. Due to space limitation, the detailed proof of Lemmas and Theorems are moved to the long version~\cite{DBLP:journals/corr/abs-1803-05621}.

\subsection{Partition}
First, we give the following definition:
\begin{definition}\label{def:pi}
Define $\pi = [\phi_1(\cdot),\ldots,\phi_p(\cdot)]$. We call $\pi$ a \emph{partition} w.r.t. $P(\cdot)$, if $F(\w) = \frac{1}{p}\sum_{k=1}^p \phi_k(\w)$ and each
 $\phi_k(\cdot)~(k=1,\ldots,p)$ is $\mu_k$-strongly convex and $L_k$-smooth~($\mu_k, L_k>0$). Here, $P(\cdot)$ is defined in~(\ref{eq:obj}) and $F(\cdot)$ is defined in~(\ref{eq:proxgd}). We denote $A(P) = \{\pi|\pi$ is a partition w.r.t. $P(\cdot)\}$.
\end{definition}

\begin{remark}
Here, $\pi$ is an ordered sequence of functions. In particular, if we construct another partition $\pi{'}$ by permuting $\phi_i(\cdot)$ of $\pi$, we consider them to be two different partitions. Furthermore, two functions $\phi_i(\cdot), \phi_j(\cdot)~(i\neq j)$ in $\pi$ can be the same. Two partitions
$\pi_1 = [\phi_1(\cdot),\ldots,\phi_p(\cdot)]$, $\pi_2 = [\psi_1(\cdot),\ldots,\psi_p(\cdot)]$ are considered to be equal, i.e., $\pi_1 = \pi_2$, if and only if $\phi_k(\w) = \psi_k(\w) (k=1,\ldots,p), \forall \w$.
\end{remark}

For any partition $\pi = [\phi_1(\cdot),\ldots,\phi_p(\cdot)]$ w.r.t. $P(\cdot)$, we construct new functions $P_k(\cdot;\cdot)$ as follows:
\begin{align}
P_k(\w;\a)  =  \phi_k(\w;\a) + R(\w),k=1,\ldots,p \label{eq:pk}
\end{align}
where $\phi_k(\w;\a) =  \phi_k(\w) + G_k(\a)^T\w$, $G_k(\a) = \nabla F(\a) - \nabla \phi_k(\a)$, and $\w, \a \in \RB^d$.


In particular, given a data partition $D_1, D_2,\ldots,D_p$ of the training set $D$, let $F_k(\w) = \frac{1}{|D_k|}\sum_{i\in D_k} f_i(\w)$ which is also called the \emph{local loss function}. Assume each $F_k(\cdot)$ is strongly convex and smooth, and $F(\w) = \frac{1}{p}\sum_{k=1}^p F_k(\w)$. Then, we can find that $\pi = [F_1(\cdot),\ldots,F_p(\cdot)]$ is a partition w.r.t. $P(\cdot)$. By taking expectation on $\v_{k,m}$ defined in Algorithm~\ref{alg:proxsco}, we obtain $\EB[\v_{k,m}|\u_{k,m}] = \nabla F_k(\u_{k,m}) + G_k(\w_t)$. According to the theory in~\cite{DBLP:journals/siamjo/Xiao014}, in the inner iterations of pSCOPE, each worker tries to optimize the local objective function $P_k(\w;\w_t)$ using proximal SVRG with initialization $\w = \w_t$ and training data $D_k$, rather than optimizing $F_k(\w)+R(\w)$. Then we call such a $P_k(\w;\a)$ the \emph{local objective function} w.r.t. $\pi$. Compared to the subproblem of PROXCOCOA$+$~(equation~(2) in~\cite{DBLP:journals/corr/SmithFJJ15}), $P_k(\w;\a)$ is more simple and there is no hyperparameter in it.

\subsection{Good Partition}

In general, the data distribution on each worker is different from the distribution of the whole training set. Hence, there exists a gap between each local optimal value and the global optimal value. Intuitively, the whole learning algorithm has slow convergence rate or cannot even converge if this gap is too large.


\begin{definition}\label{def:Local-Global Gap}
For any partition $\pi$ w.r.t. $P(\cdot)$, we define the \emph{Local-Global Gap} as
\begin{align}
l_\pi(\a) = P(\w^*) - \frac{1}{p}\sum_{k=1}^p P_k(\w_k^*(\a);\a), \nonumber
\end{align}
where $\w_k^*(\a) = \mathop{\arg\min}_{\w} P_k(\w;\a)$.
\end{definition}

We have the following properties of Local-Global Gap:

\begin{lemma}\label{lem:dual}
$\forall \pi \in A(P)$, $l_\pi(\a) = P(\w^*) + \frac{1}{p}\sum_{k=1}^p H_k^*(-G_k(\a)) \geq l_\pi(\w^*) = 0, \forall \a$, where $H_k^*(\cdot)$ is the conjugate function of $\phi_k(\cdot) + R(\cdot)$.
\end{lemma}

\begin{theorem}\label{theorem:L1 qualified pi}
Let $R(\w) = \|\w\|_1$. $\forall \pi \in A(P)$, there exists a constant $\gamma< \infty$ such that $l_\pi(\a) \leq \gamma \|\a - \w^*\|^2, \forall \a$.
\end{theorem}

The result in Theorem~\ref{theorem:L1 qualified pi} can be easily extended to smooth regularization which can be found in the long version~\cite{DBLP:journals/corr/abs-1803-05621}.

According to Theorem~\ref{theorem:L1 qualified pi}, the local-global gap can be bounded by $\gamma\|\a - \w^*\|^2$. Given a specific $\a$, the smaller $\gamma$ is, the smaller the local-global gap will be. Since the constant $\gamma$ only depends on the partition $\pi$, intuitively $\gamma$ can be used to evaluate the \emph{goodness} of a partition $\pi$. We define a \emph{good partition} as follows:
\begin{definition}\label{def:good}
We call $\pi$ a $(\epsilon,\xi)$-\emph{good partition} w.r.t. $P(\cdot)$ if $\pi \in A(P)$ and
\begin{align}
\gamma(\pi;\epsilon) \overset{\triangle}{=} \mathop{sup}_{\|\a-\w^*\|^2\geq \epsilon} \frac{l_\pi(\a)}{\|\a - \w^*\|^2} \leq \xi .
\end{align}
\end{definition}

In the following, we give the bound of $\gamma(\pi;\epsilon)$.
%
%
%

\begin{lemma}\label{lemma: Good partition}
Assume $\pi=[F_1(\cdot),\ldots,F_p(\cdot)]$ is a partition w.r.t. $P(\cdot)$, where $F_k(\w) =  \frac{1}{|D_k|}\sum_{i \in D_k}f_i(\w)$ is the local loss function, each $f_i(\cdot)$ is Lipschitz continuous with bounded domain and sampled from some unknown distribution $\PB$.
If we assign these $\{f_i(\cdot)\}$ uniformly to each worker, then with high probability, $\gamma(\pi; \epsilon) \leq  \frac{1}{p}\sum_{k=1}^p \mathcal{O}(1/(\epsilon\sqrt{|D_k|}))$. Moreover, if $l_\pi(\a)$ is convex w.r.t. $\a$, then $\gamma(\pi; \epsilon) \leq  \frac{1}{p}\sum_{k=1}^p \mathcal{O}(1/\sqrt{\epsilon|D_k|})$. Here we ignore the $\log$ term and dimensionality $d$.
\end{lemma}

For example, in Lasso regression, it is easy to get that the corresponding local-global gap $l_\pi(\a)$ is convex according to Lemma \ref{lem:dual} and the fact that $G_k(\a)$ is an affine function in this case.

Lemma \ref{lemma: Good partition} implies that as long as the size of training data is large enough, $\gamma(\pi; \epsilon)$ will be small and $\pi$ will be a good partition. Please note that the \emph{uniformly} here means each $f_i(\cdot)$ will be assigned to one of the $p$ workers and each worker has the equal probability to be assigned. We call the partition resulted from uniform assignment \emph{uniform partition} in this paper. With uniform partition, each worker will have almost the same number of instances. As long as the size of training data is large enough, uniform partition is a good partition.

\section{Convergence of Proximal SCOPE}

In this section, we will prove the convergence of Algorithm~\ref{alg:proxsco} for proximal SCOPE~(pSCOPE) using the results in Section~\ref{sec:FDP}.



\begin{theorem}\label{theorem:qualified segementation}
Assume $\pi= [F_1(\cdot),\ldots,F_p(\cdot)]$ is a $(\epsilon,\xi)$-good partition w.r.t. $P(\cdot)$. For convenience, we set $\mu_k = \mu, L_k = L, k=1,2\ldots,p$. If $\|\w_t - \w^*\|^2 \geq \epsilon$, then
\begin{align}
  & \EB\|\w_{t+1} - \w^*\|^2 \leq [(1-\mu\eta + 2L^2\eta^2)^M+\frac{2L^2\eta + 2\xi}{\mu - 2L^2\eta}]\|\w_t - \w^*\|^2. \nonumber
\end{align}
\end{theorem}
\vspace{-0.2cm}
Because smaller $\xi$ means better partition and the partition $\pi$ corresponds to data partition in Algorithm~\ref{alg:proxsco}, we can see that \emph{better data partition implies faster convergence rate}.

\begin{corollary}\label{corollary:good segementation}
Assume $\pi= [F_1(\cdot),\ldots,F_p(\cdot)]$ is a $(\epsilon,\frac{\mu}{8})$-good partition w.r.t. $P(\cdot)$. For convenience, we set $\mu_k = \mu, L_k = L, k=1,2\ldots,p$. If $\|\w_t - \w^*\|^2 \geq \epsilon$, taking $\eta = \frac{\mu}{12L^2}$, $M = 20\kappa^2$, where $\kappa = \frac{L}{\mu}$ is the conditional number, then we have $\EB \|\w_{t+1} - \w^*\|^2 \leq \frac{3}{4}\|\w_t - \w^*\|^2$. To get the $\epsilon$-suboptimal solution, the computation complexity of each worker is $O((n/p+\kappa^2) \log(\frac{1}{\epsilon}))$.
\end{corollary}

\begin{corollary}\label{corollary:proxsvrg}
When $p = 1$, which means we only use one worker, pSCOPE degenerates to proximal SVRG~\cite{DBLP:journals/siamjo/Xiao014}. Assume $F(\cdot)$ is $\mu$-strongly convex ($\mu > 0$) and $L$-smooth. Taking $\eta = \frac{\mu}{6L^2}$, $M = 13\kappa^2$, we have $\EB \|\w_{t+1} - \w^*\|^2 \leq \frac{3}{4}\|\w_t - \w^*\|^2$. To get the $\epsilon$-optimal solution, the computation complexity is $O((n+\kappa^2) \log(\frac{1}{\epsilon}))$.
\end{corollary}

We can find that pSCOPE has a linear convergence rate if the partition is $(\epsilon, \xi)$-good, which implies pSCOPE is computation efficient and we need $T=O(\log(\frac{1}{\epsilon}))$ outer iterations to get a $\epsilon$-optimal solution. For all inner iterations, each worker updates $\u_{k,m}$ without any communication. Hence, the communication cost is $O(\log(\frac{1}{\epsilon}))$, which is much smaller than the mini-batch based strategy with $O(n)$ communication cost for each epoch~\cite{DBLP:journals/corr/LiXZL16,DBLP:journals/corr/HuoGH16,DBLP:conf/aaai/MengCYWML17}.

Furthermore, in the above theorems and corollaries, we only assume that the local loss function $F_k(\cdot)$ is strongly convex. We do not need each $f_i(\cdot)$ to be convex. Hence, $M = O(\kappa^2)$ and it is weaker than the assumption in proximal SVRG~\cite{DBLP:journals/siamjo/Xiao014} whose computation complexity is $O((n+\kappa) \log(\frac{1}{\epsilon}))$ when $p=1$. In addition, without convexity assumption for each $f_i(\cdot)$, our result for the degenerate case $p=1$ is consistent with that in~\cite{DBLP:conf/icml/Shalev-Shwartz16}.

\section{Handle High Dimensional Sparse Data}\label{sec:sparseData}
For the cases with high dimensional sparse data, we propose \emph{recovery strategy} to reduce the cost of proximal mapping so that it can accelerate the training procedure. Here, we adopt the widely used linear model with elastic net~\cite{Zou05regularizationand} as an example for illustration, which can be formulated as follows: $\mathop{\min}_{\w} P(\w) := \frac{1}{n} \sum_{i=1}^{n} h_i(\x_i^T\w) + \frac{\lambda_1}{2}\|\w\|^2 + \lambda_2\|\w\|_1$,
where $h_i:\RB \rightarrow \RB$ is the loss function. We assume many instances in $\{\x_i \in \RB^d | i\in [n]\}$ are sparse vectors and let $C_i = \{j|x_i^{(j)} \neq 0\}$.

Proximal mapping is unacceptable when the data dimensionality $d$ is too large, since we need to execute the conditional statements $O(Md)$ times which is time consuming. Other methods, like proximal SGD and proximal SVRG, also suffer from this problem.

Since $z^{(j)}$ is a constant during the update of local parameter~$\u_{k,m}$, we will design a \emph{recovery strategy} to recover it when necessary. More specifically, in each inner iteration, with the random index $s = i_{k,m}$, we only \emph{recover} $u^{(j)}$ to calculate the inner product $\x_s^T\u_{k,m}$ and update $u^{(j)}_{k,m}$ for $j\in C_s$. For those $j\notin C_s$, we do not immediately update $u^{(j)}_{k,m}$. The basic idea of these recovery rules is: for some coordinate $j$, we can calculate $u_{k,m_2}^{(j)}$ directly from $u_{k,m_1}^{(j)}$, rather than doing iterations from $m=m_1$ to $m_2$. Here, $0\leq m_1<m_2\leq M$. At the same time, the new algorithm is totally equivalent to Algorithm~\ref{alg:proxsco}. It will save about $O(d(m_2-m_1)(1-\rho))$ times of conditional statements, where $\rho$ is the sparsity of $\{\x_i \in \RB^d | i\in [n]\}$. This reduction of computation is significant especially for high dimensional sparse training data. Due to space limitation, the complete rules are moved to the long version~\cite{DBLP:journals/corr/abs-1803-05621}. Here we only give one case of our recovery rules in Lemma 3.

\begin{lemma}
(\textbf{Recovery Rule})
We define the sequence $\{\alpha_q\}$ as: $\alpha_0 = 0$ and for $q = 1,2,\ldots$, $\alpha_{q} = \sum_{i=1}^{q}(1-\lambda_1\eta)^{i-1}/(1-\lambda_1\eta)^q$.
For the coordinate $j$ and constants $m_1, m_2$, if $j \notin C_{i_{k,m}}$ for any $m\in [m_1, m_2-1]$. If $|z^{(j)}| < \lambda_2, u_{k,m_1}^{(j)} > 0$, then the relation between $u_{k,m_1}^{(j)}$ and $u_{k,m_2}^{(j)}$ can be summarized as follows: define $q_0$ which satisfies $\alpha_{q_0}\eta(z^{(j)}+\lambda_2)\leq u_{k,m_1}^{(j)} < \alpha_{q_0+1}\eta(z^{(j)}+\lambda_2),$
\begin{enumerate}
\item If $m_2 - m_1 \leq q_0$, then $u_{k,m_2}^{(j)} = (1-\lambda_1\eta)^{m_2-m_1}[u_{k,m_1}^{(j)} - \alpha_{m_2-m_1}\eta(z^{(j)}+\lambda_2)].$

\item If $m_2 - m_1 > q_0$, then $u_{k,m_2}^{(j)} = 0.$
        \end{enumerate}
\end{lemma}

\section{Experiment}
We use two sparse learning models for evaluation. One is logistic regression~(LR) with elastic net~\cite{Zou05regularizationand}: $P(\w) = \frac{1}{n}\sum_{i=1}^{n}\log(1+e^{-y_i\x_i^T\w}) + \frac{\lambda_1}{2}\|\w\|^2 + \lambda_2\|\w\|_1$.
The other is Lasso regression~\cite{Tibshirani94regressionshrinkage}: $P(\w) = \frac{1}{2n}\sum_{i=1}^{n}(\x_i^T\w-y_i)^2 + \lambda_2\|\w\|_1$. All experiments are conducted on a cluster of multiple machines. The CPU for each machine has 12 Intel E5-2620 cores, and the memory of each machine is 96GB. The machines are connected by 10GB Ethernet. Evaluation is based on four datasets in Table 1: cov, rcv1, avazu, kdd2012. All of them can be downloaded from LibSVM website~\footnote{\url{https://www.csie.ntu.edu.tw/~cjlin/libsvmtools/datasets/}}.
\begin{table}[ht]
\caption{Datasets}
\begin{center}
\small
\begin{tabular}{|l|l|l|l|c|}
  \hline
  ~     & \#instances   & \#features   & $\lambda_1$ & $\lambda_2$ \\ \hline
  cov   & 581,012     & 54         & $10^{-5}$   & $10^{-5}$   \\ \hline
  rcv1  & 677,399     & 47,236     & $10^{-5}$   & $10^{-5}$   \\ \hline
  avazu & 23,567,843  & 1,000,000  & $10^{-7}$   & $10^{-5}$   \\ \hline
  kdd2012 & 119,705,032 & 54,686,452 & $10^{-8}$   & $10^{-5}$   \\
  \hline
\end{tabular}
\end{center}
\end{table}

\subsection{Baselines}
We compare our pSCOPE with six representative baselines: proximal gradient descent based method FISTA~\cite{DBLP:journals/siamis/BeckT09}, ADMM type method \mbox{DFAL}~\cite{DBLP:conf/icml/AybatWI15}, newton type method \mbox{mOWL-QN}~\cite{DBLP:conf/icml/GongY15}, proximal SVRG based method AsyProx-SVRG~\cite{DBLP:conf/aaai/MengCYWML17}, proximal SDCA based method PROXCOCOA$+$~\cite{DBLP:journals/corr/SmithFJJ15}, and distributed block
coordinate descent DBCD~\cite{DBLP:journals/jmlr/MahajanKS17}. FISTA and mOWL-QN are serial. We design distributed versions of them, in which workers distributively compute the gradients and then master gathers the gradients from workers for parameter update.

All methods use 8 workers. One master will be used if necessary. Unless otherwise stated, all methods except DBCD and PROXCOCOA$+$ use the same data partition, which is got by uniformly assigning each instance to each worker~(uniform partition). Hence, different workers will have almost the same number of instances. This uniform partition strategy satisfies the condition in Lemma~\ref{lemma: Good partition}. Hence, it is a \emph{good} partition. DBCD and PROXCOCOA$+$ adopt a coordinate distributed strategy to partition the data.

\subsection{Results}

The convergence results of LR with elastic net and Lasso regression are shown in Figure 1. DBCD is too slow, and hence we will separately report the time of it and pSCOPE when they get $10^{-3}$-suboptimal solution in Table 2. AsyProx-SVRG is slow on the two large datasets avazu and kdd2012, and hence we only present the results of it on the datasets cov and rcv1. From Figure 1 and Table 2, we can find that pSCOPE outperforms all the other baselines on all datasets.
\begin{figure*}[thb]
\centering
\subfigure[LR with elastic net]{
    \includegraphics[width=1.35in]{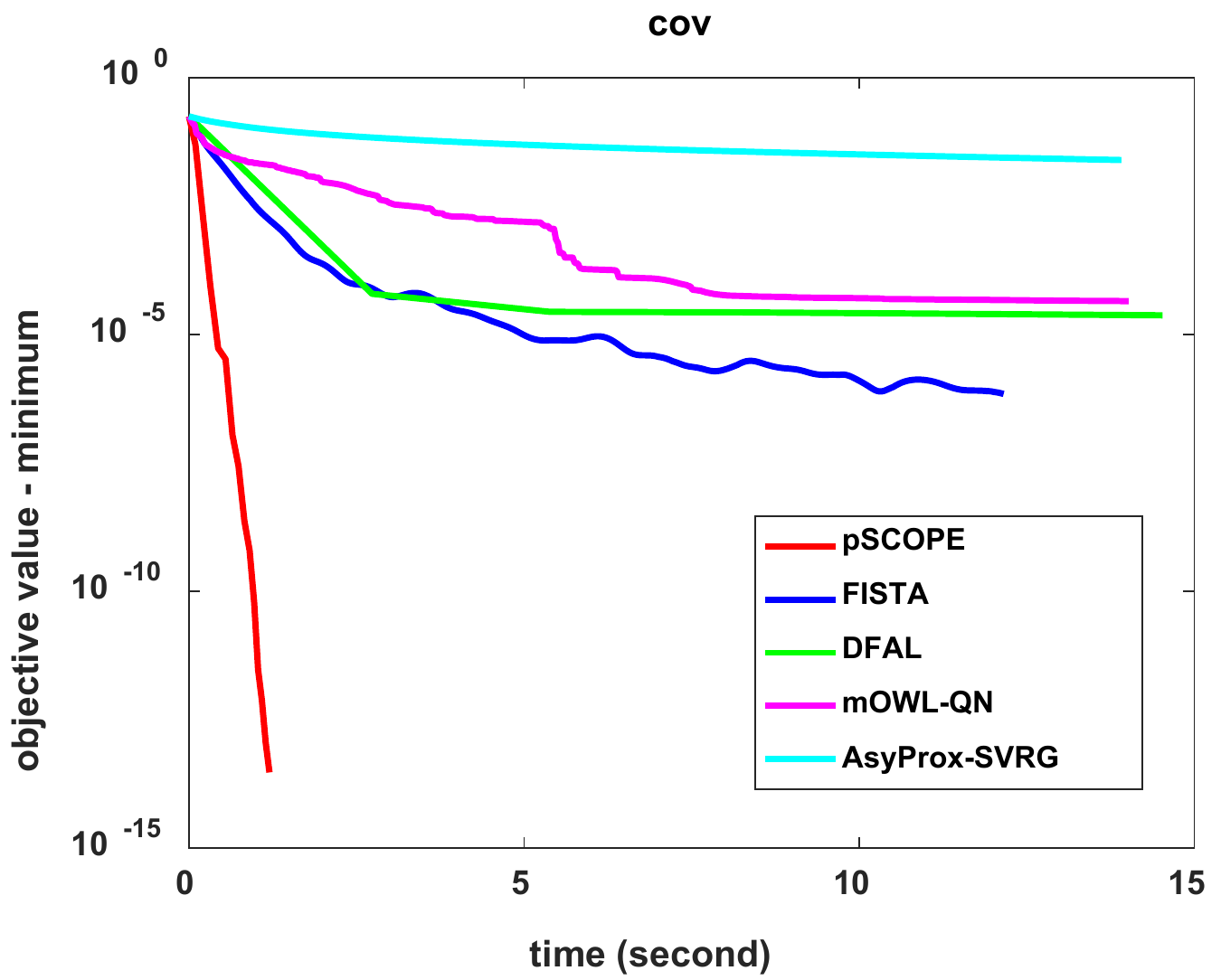}
    \includegraphics[width=1.35in]{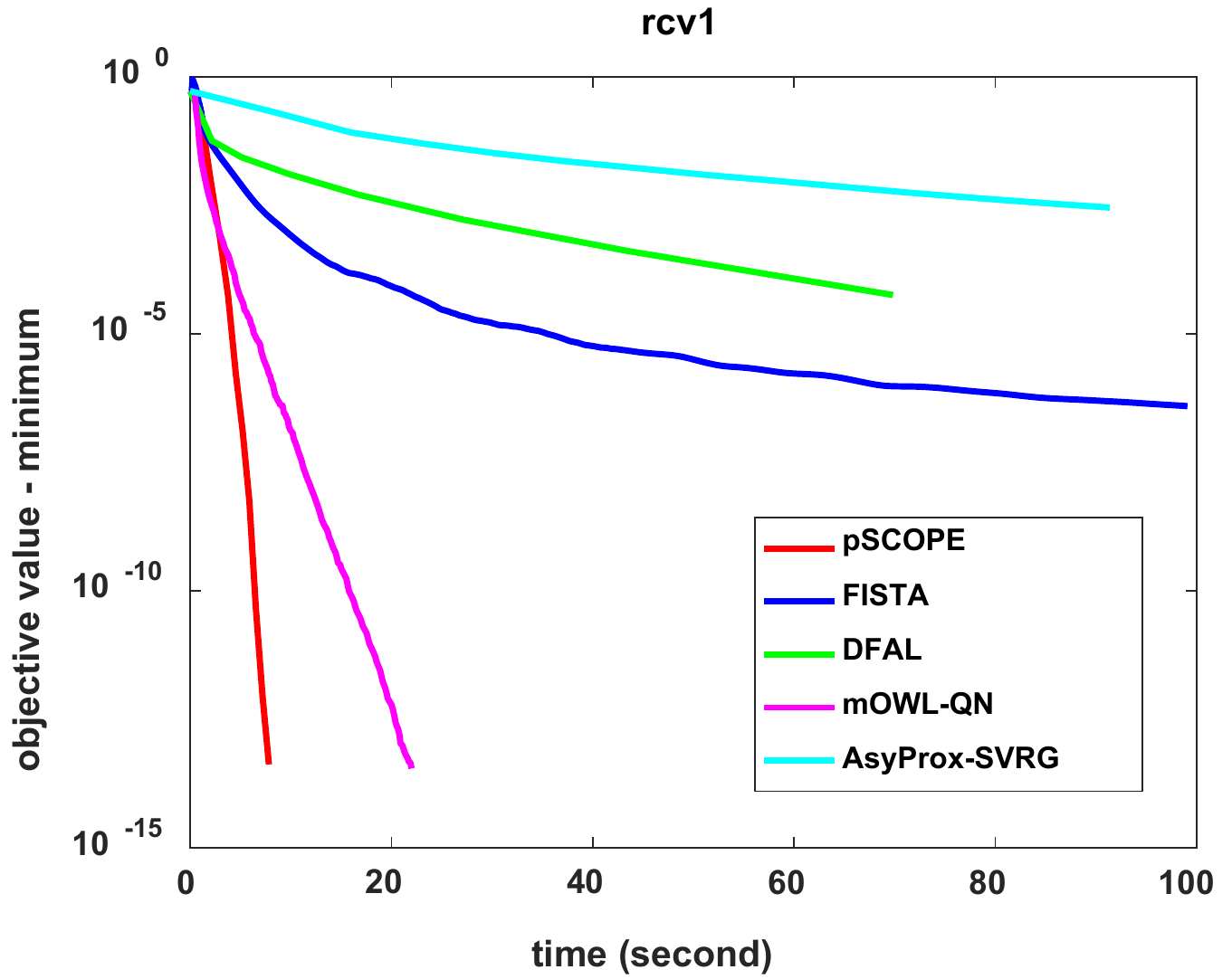}
    \includegraphics[width=1.35in]{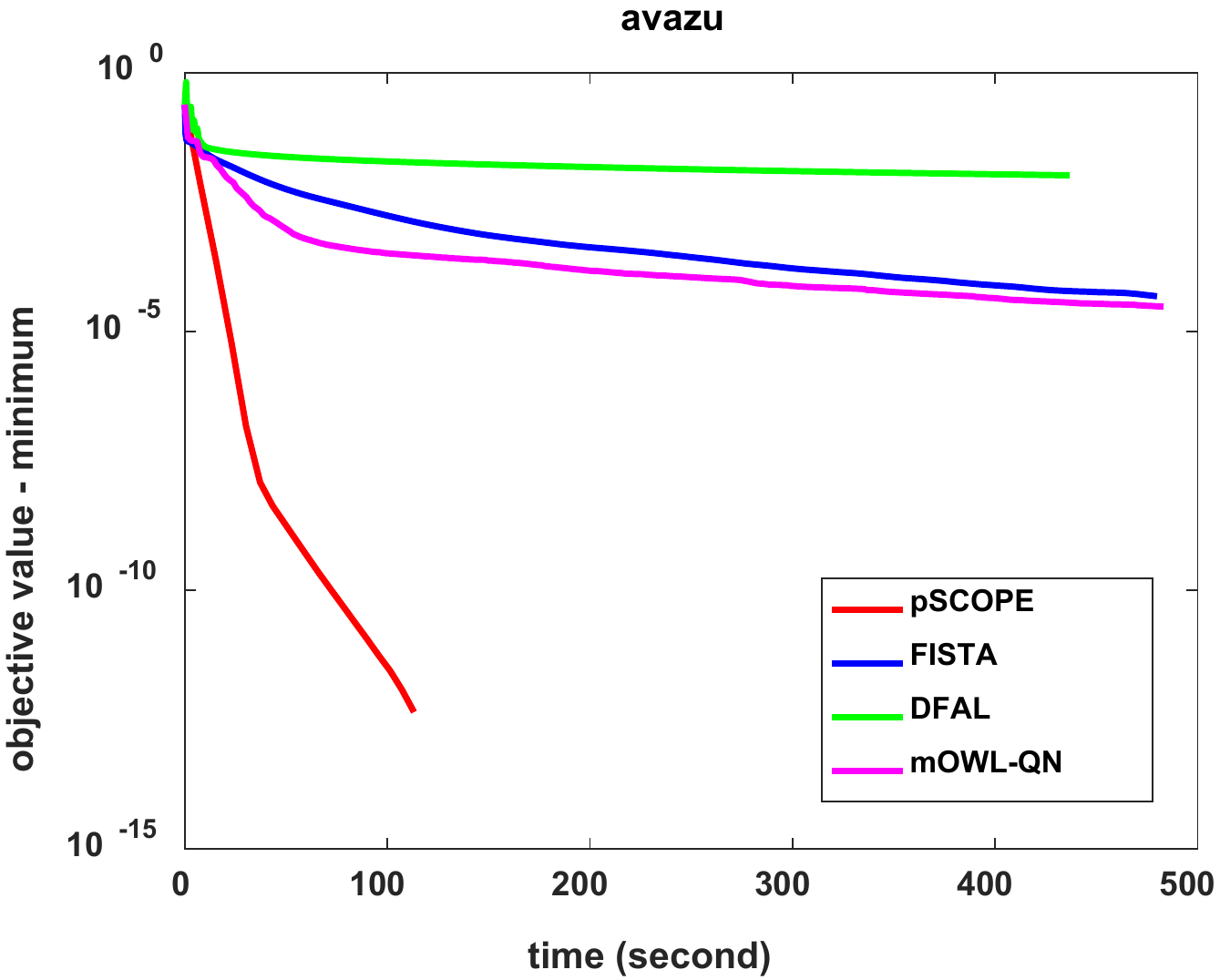}
    \includegraphics[width=1.35in]{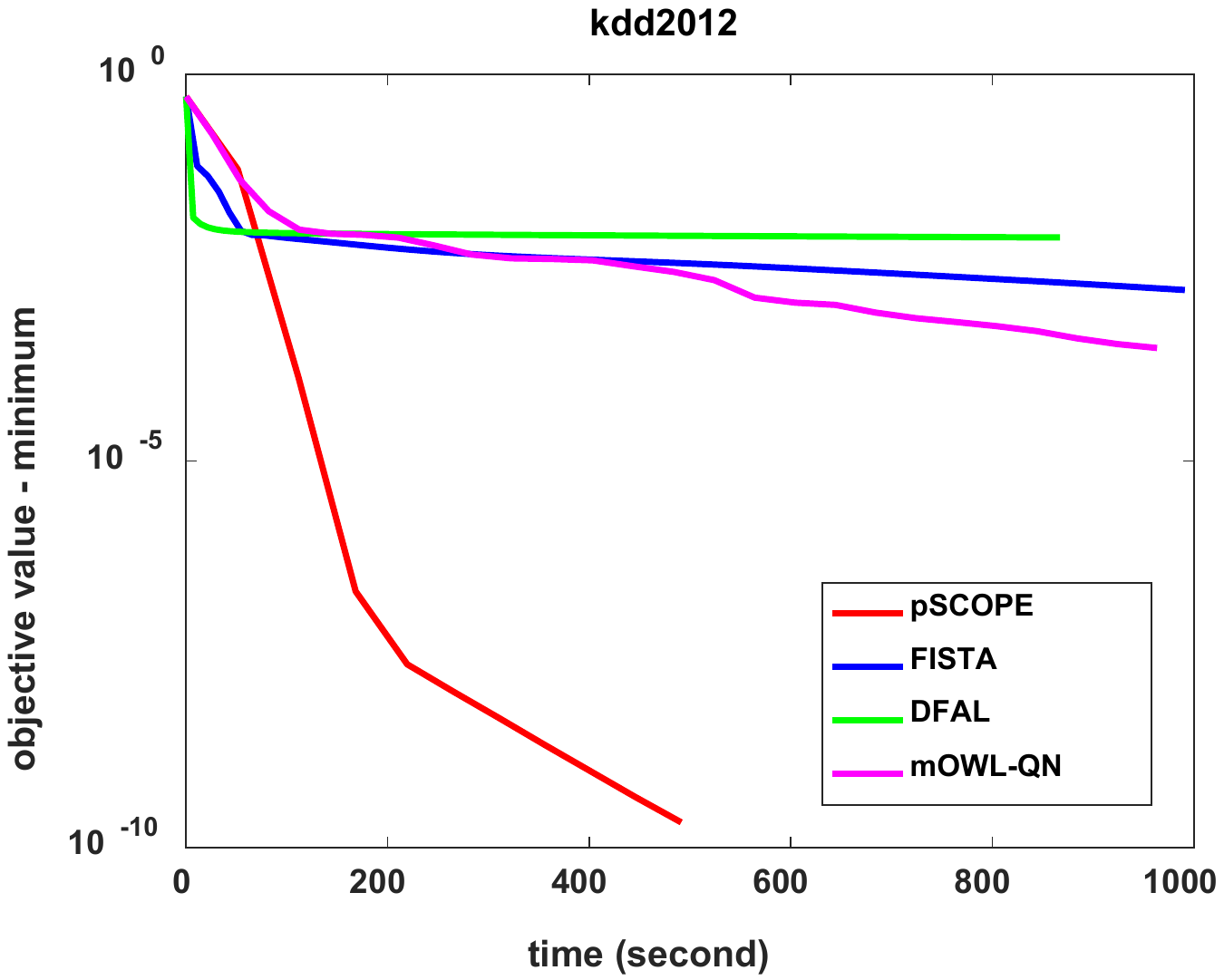}}
\subfigure[Lasso regression]{
    \includegraphics[width=1.35in]{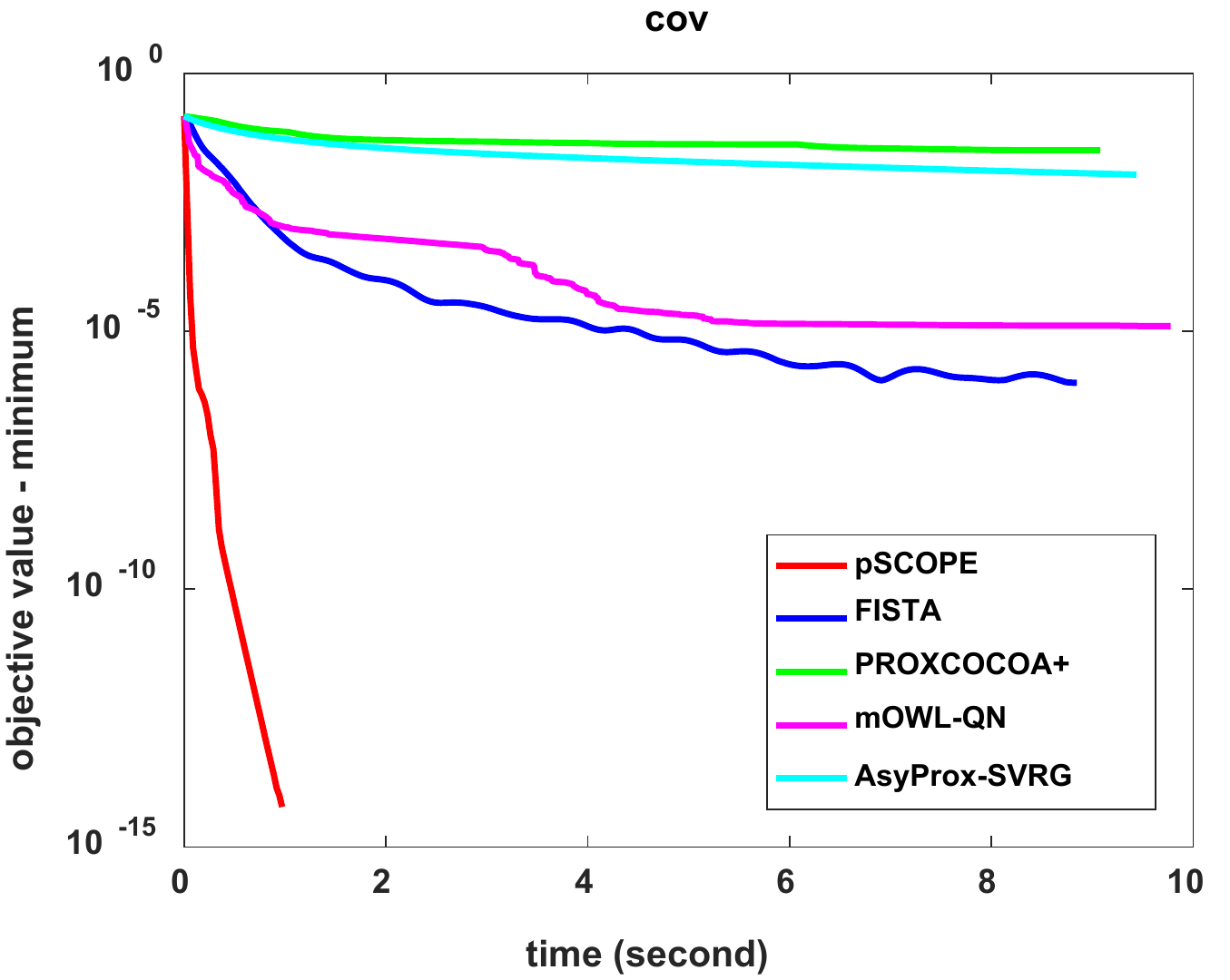}
    \includegraphics[width=1.35in]{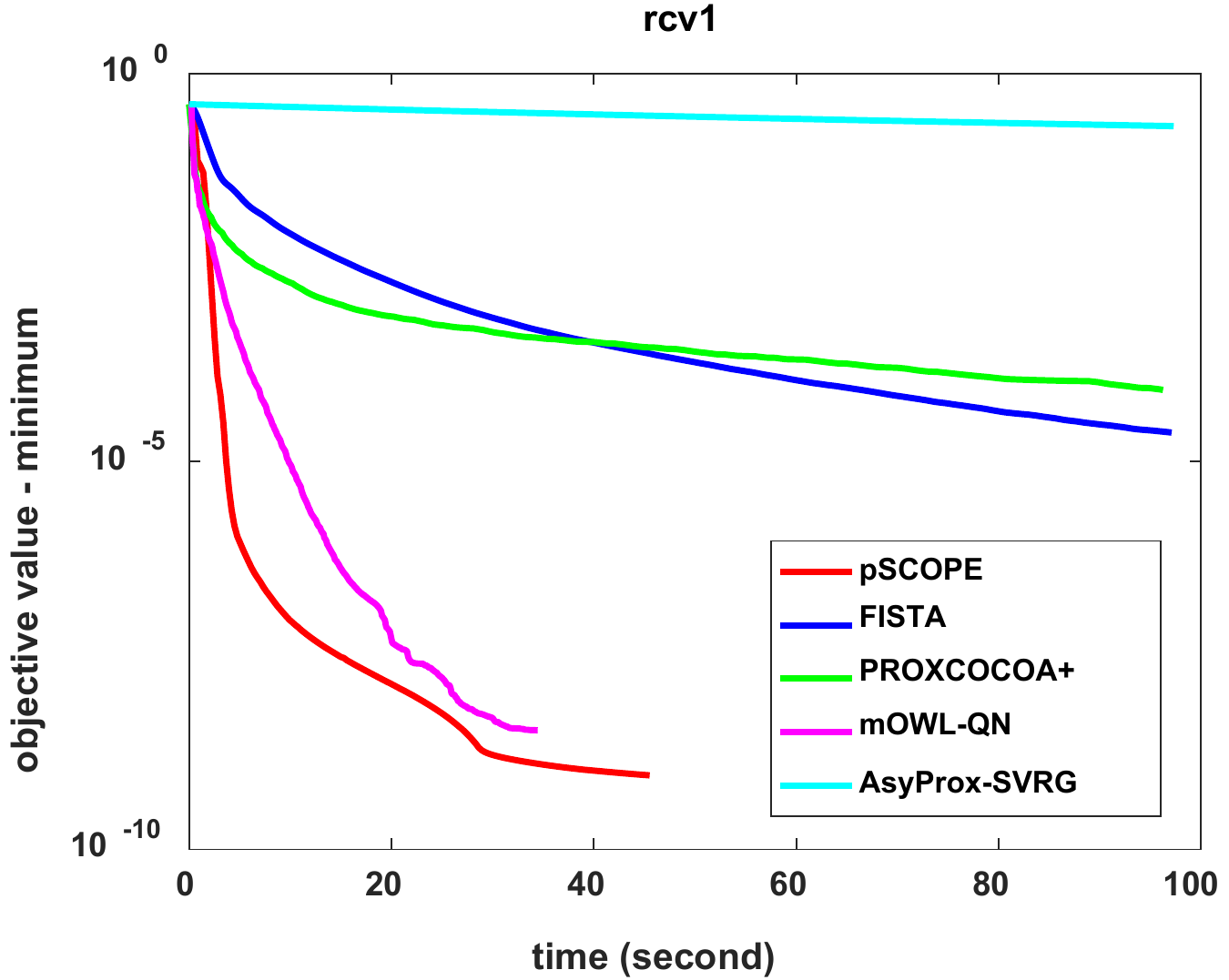}
    \includegraphics[width=1.35in]{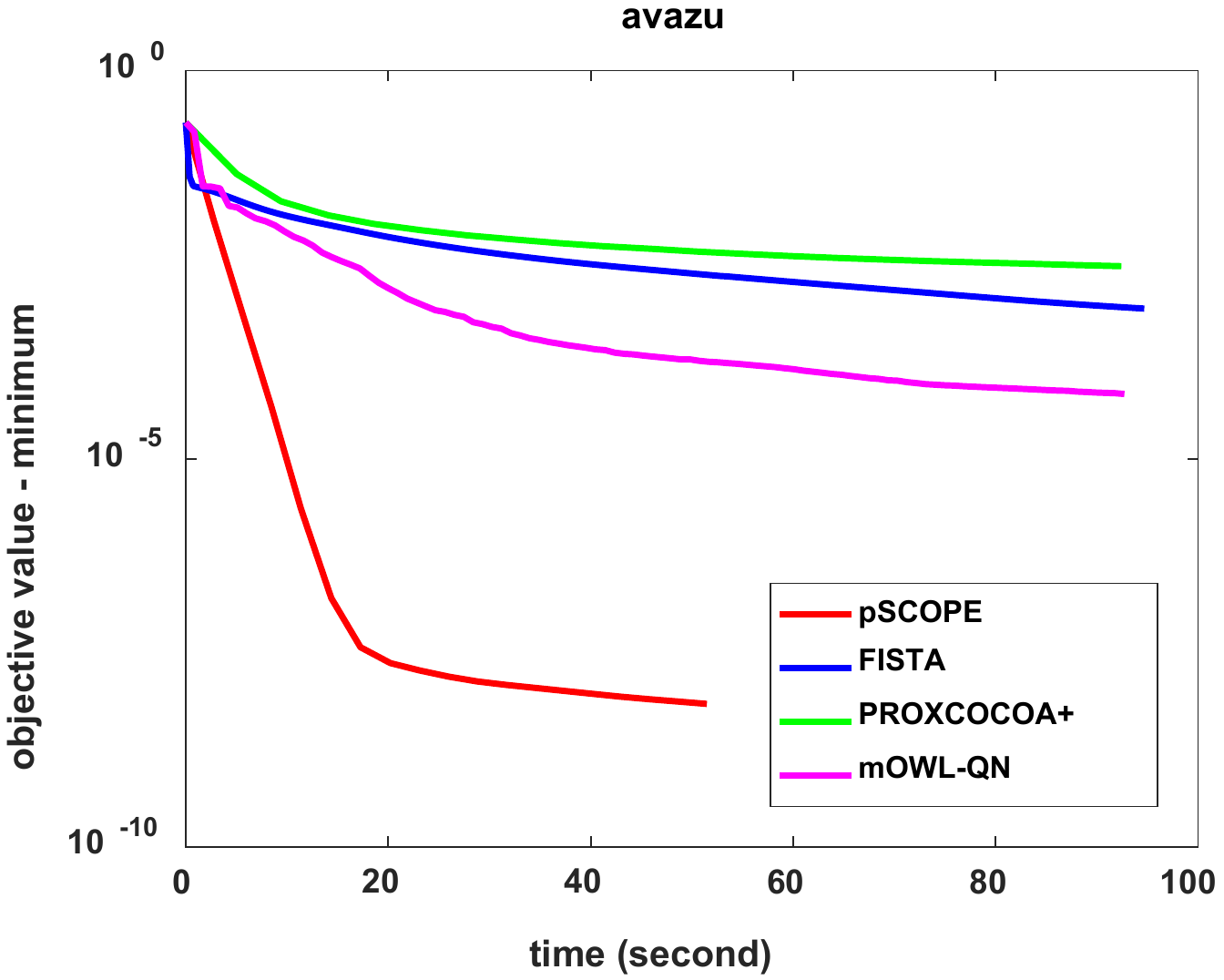}
    \includegraphics[width=1.35in]{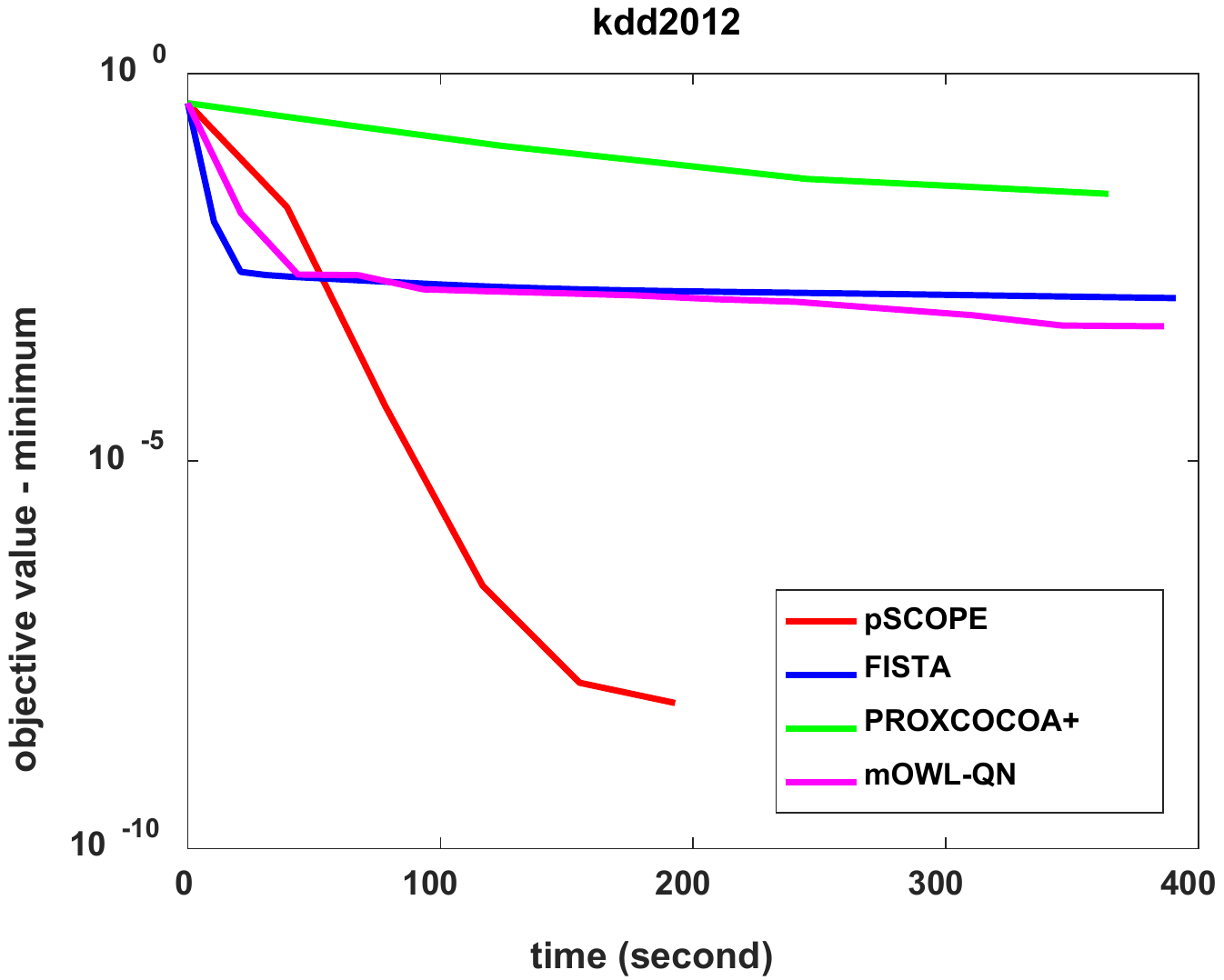}}
\caption{Evaluation with baselines on two models.}
\end{figure*}

\begin{table}[thb]
  \centering
  \caption{Time comparison~(in second) between pSCOPE and DBCD.}
  \begin{tabular}{|c|c|c|c|}
  \hline
  \multicolumn{2}{|c|}{} & pSCOPE & {DBCD}\\\hline
     & cov  & 0.32  & 822 \\\cline{2-4}
   {LR}
   & rcv1 & 3.78 & $>1000$  \\\hline
   & cov & 0.06 & 81.9  \\\cline{2-4}
   {Lasso}
            & rcv1 & 3.09 & $>1000$ \\\hline
  \end{tabular}
\end{table}

\subsection{Speedup}
We also evaluate the speedup of pSCOPE on the four datasets for LR. We run pSCOPE and stop it when the gap $P(\w) - P(\w^*) \leq 10^{-6}$. The speedup is defined as: ${\mbox{Speedup}} = (\mbox{Time using one worker})/(\mbox{Time using } p \mbox{ workers})$.
We set $p = 1,2,4,8$. The speedup results are in Figure 2~(a). We can find that \mbox{pSCOPE} gets promising speedup.
\begin{figure}[thb]
  \centering
  \subfigure[Speedup of pSCOPE]{\includegraphics[width=1.6in]{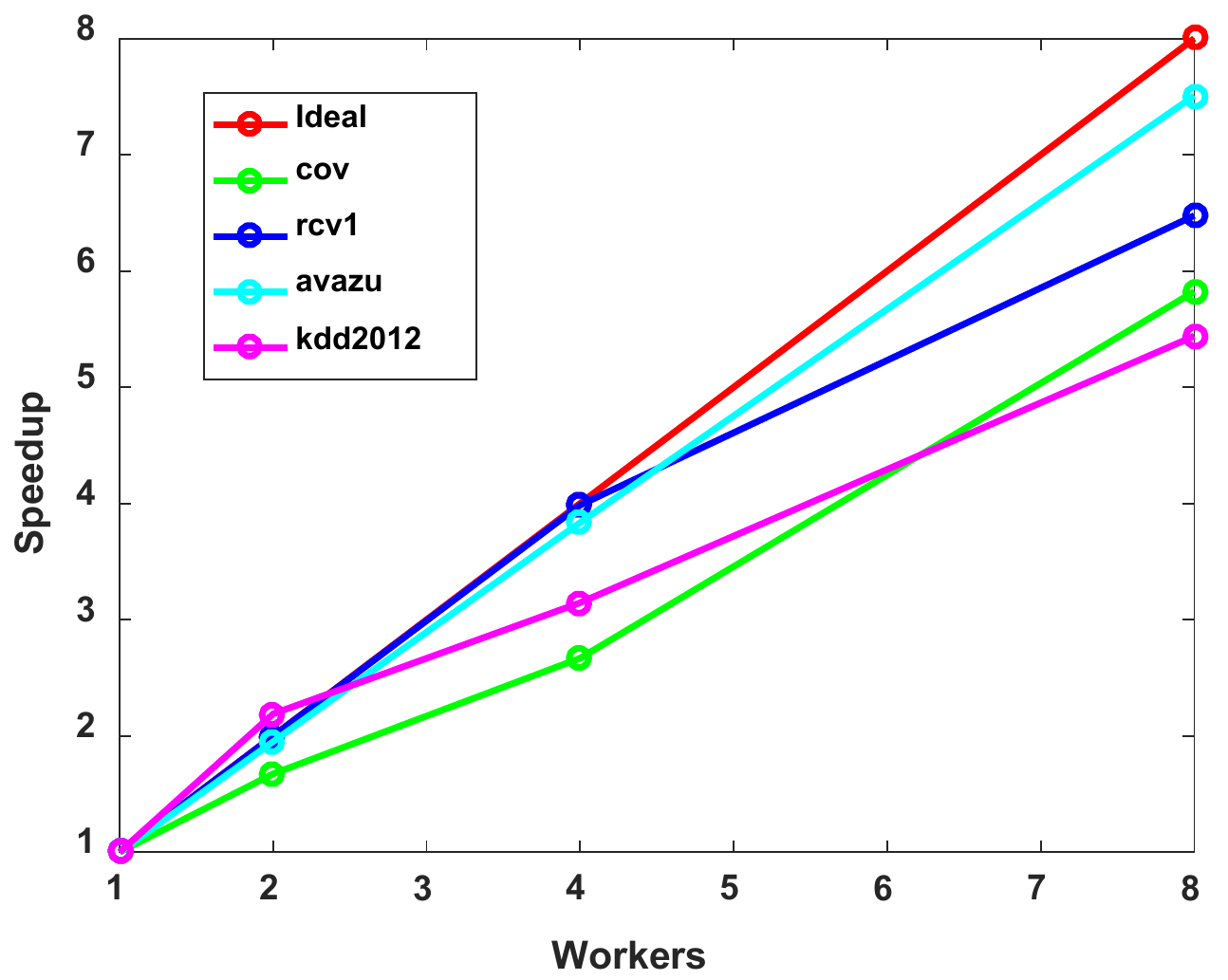}}
  \subfigure[Effect of data partition]{\includegraphics[width=1.6in]{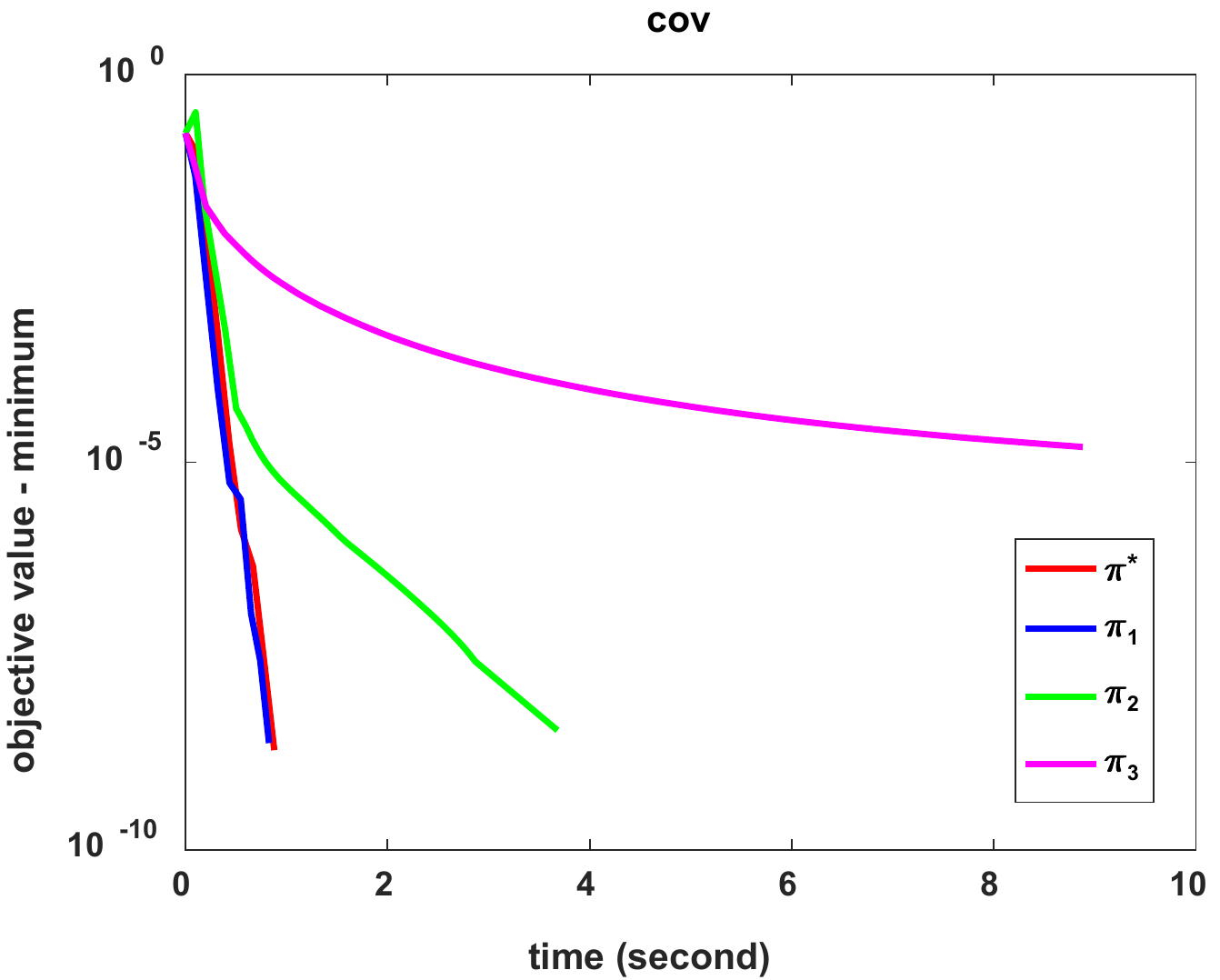}
       \includegraphics[width=1.6in]{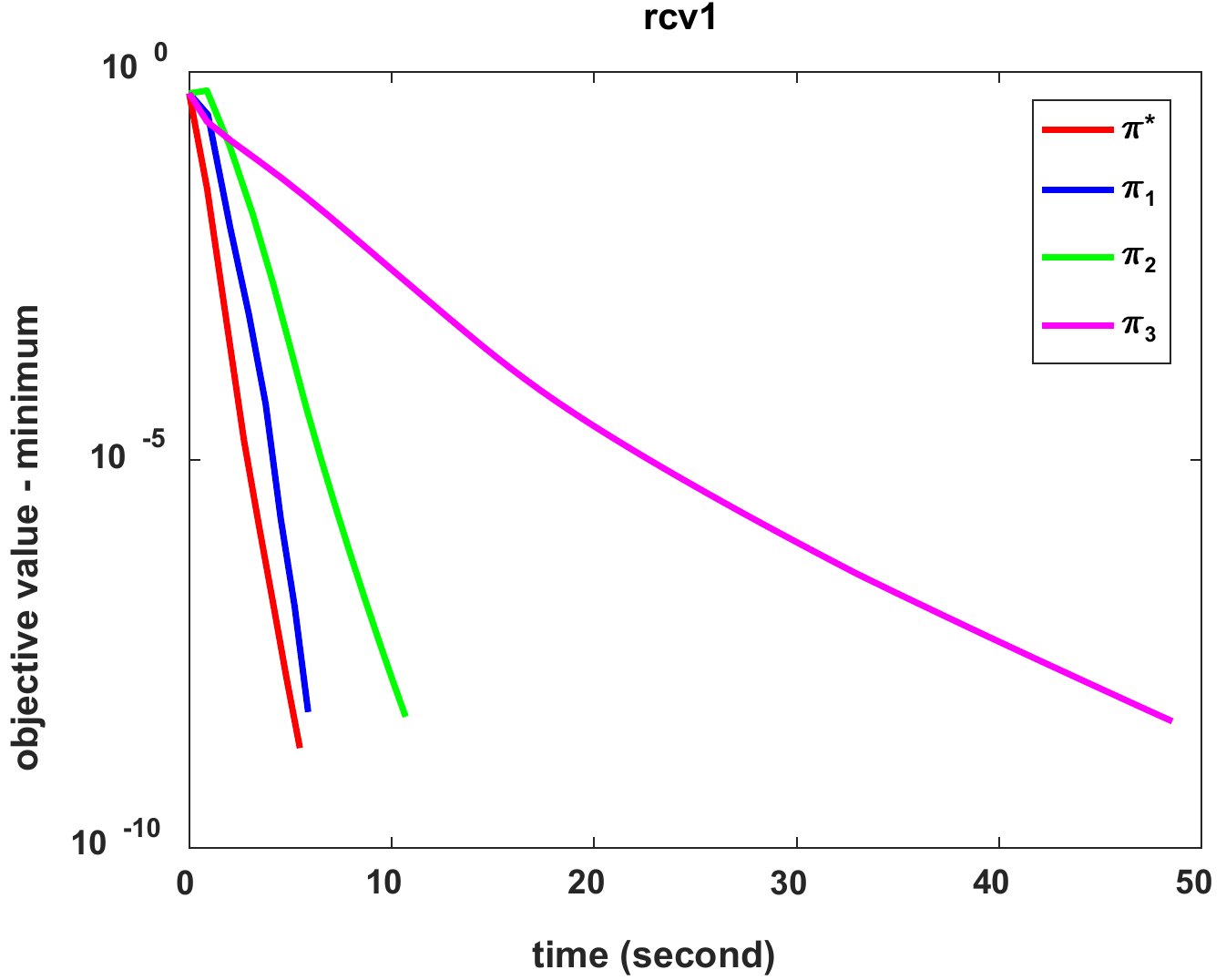}}
  \caption{Speedup and effect of data partition}
\end{figure}

\subsection{Effect of Data Partition}

We evaluate pSCOPE under different data partitions. We use two datasets cov and rcv1 for illustration, since they are balanced datasets which means the number of positive instances is almost the same as that of negative instances. For each dataset, we construct four data partitions: $\pi^*$~(Each worker has access to the whole data), $\pi_1$~(Uniform partition); $\pi_2$~(75\% positive instances and 25\% negative instances are on the first 4 workers, and other instances are on the last 4 workers), $\pi_3$~(All positive instances are on the first 4 workers, and all negative instances are on the last 4 workers).

The convergence results are shown in Figure 2~(b). We can see that data partition does affect the convergence of pSCOPE. The partition $\pi^*$  achieves the best performance, which verifies the theory in this paper~\footnote{The proof of that $\pi^*$ is the best partition and $\gamma(\pi^*,0) = 0$ is in the appendix}. The performance of uniform partition $\pi_1$ is similar to that of the best partition $\pi^*$, and is better than the other two data partitions. In real applications with large-scale dataset, it is impractical to assign each worker the whole dataset. Hence, we prefer to choose uniform partition $\pi_1$ in real applications, which is also adopted in above experiments of this paper.

\section{Conclusion}

In this paper, we propose a novel method, called pSCOPE, for distributed sparse learning. Furthermore, we theoretically analyze how the data partition affects the convergence of pSCOPE. pSCOPE is both communication and computation efficient. Experiments on real data show that pSCOPE can outperform other state-of-the-art methods to achieve the best performance.

\small
\bibliography{ref}
\bibliographystyle{plain}

\clearpage
\appendix

\section{Effect of Data Partition}

\subsection{Proof of Lemma \ref{lem:dual}}
$\forall \pi \in A(P)$, $l_\pi(\a) = P(\w^*) + \frac{1}{p}\sum_{k=1}^p H_k^*(-G_k(\a)) \geq l_\pi(\w^*) = 0, \forall \a$, where $H_k^*(\cdot)$ is the conjugate function of $\phi_k(\cdot) + R(\cdot)$.

\begin{proof}
According to Definition \ref{def:Local-Global Gap} and $\w_k^*(\a)$, we have
\begin{align}
	l_\pi(\a) =    & P(\w^*) - \frac{1}{p}\sum_{i=1}^p P_k(\w_k^*(\a);\a) \nonumber \\
	          =    & \frac{1}{p}\sum_{i=1}^p (P_k(\w^*;\a) - P(\w_k^*(\a);\a)) \geq 0 \nonumber
\end{align}
On the other hand, let $\a = \w^*$, then $\w_k^*(\w^*) = \mathop{\arg\min}_\w P_k(\w;\w^*)$, and
\begin{align}
    & \nabla \phi_k(\w^*;\w^*) + \zeta \nonumber \\
  = & \nabla \phi_k(\w^*) + \nabla F(\w^*) - \nabla \phi_k(\w^*) + \zeta \nonumber \\
  = & \nabla F(\w^*) + \zeta \nonumber \\
  = & \0 \nonumber
\end{align}
where $\zeta \in \partial R(\w^*)$ so that $\nabla F(\w^*) + \zeta = \0$. It implies that $\w^*=\mathop{\arg\min}_\w P_k(\w;\w^*)$. Due to the strong convexity of $P_k(\w;\w^*)$ w.r.t $\w$, we have $\w_k^*(\w^*) = \w^*$, which means $l_\pi(\w^*) = 0$.

For the dual form, according to (\ref{eq:pk}) and definition of $\w_k^*(\y)$, we have
\begin{align}
	P(\w_k^*(\a);\a) = & \mathop{\min}_\w (\phi_k(\w) + R(\w) + G_k(\a)^T\w)) \nonumber \\
	                 = & -\mathop{\max}_\w (-G_k(\a)^T\w - (\phi_k(\w) + R(\w))) \nonumber \\
	                 = & -H_k^*(-G_k(\a)) \nonumber
\end{align}
Then we have
\begin{align}
	l_\pi(\a) = P(\w^*) + \frac{1}{p}\sum_{k=1}^p H_k^*(-G_k(\a)) \nonumber
\end{align}
\end{proof}

\subsection{Proof of Theorem \ref{theorem:L1 qualified pi}}
\subsubsection{Warm up: quadratic function}
\label{sec:quadratic function}

We start with the simple quadratic function in one dimension space that $$\phi_k(w) = \frac{1}{2}m_k w^2 + b_k w + c_k,$$ and $$F(w) = \frac{1}{2}m w^2 + b w + c, R(w) = |w|,$$ where $w,m_k,b_k,c_k \in \RB, m_k>0$,
and $\frac{1}{p}\sum_{k=1}^p m_k = m, \frac{1}{p}\sum_{k=1}^p b_k = b, \frac{1}{p}\sum_{k=1}^p c_k = c$ so that $\pi = [\phi_1(\cdot),\ldots,\phi_p(\cdot)] \in A(P)$. The corresponding $P_k(\cdot;\cdot)$ is defined as
\begin{align}\label{eq:lgp_quadratic}
	P_k(w;a) = & \frac{1}{2}m_k w^2 + b_k w + c_k + (ma + b - m_ka-b_k)w + |w| \nonumber \\
	         = & \frac{1}{2}m_k w^2 + (ma+b-m_ka) w + c_k + |w|
\end{align}
Then we have the following lemma:

\begin{lemma}\label{lemma: quadratic function}
With $\phi_k(\cdot), F(\cdot), R(\cdot)$ defined above,
\begin{align}
P(w^*) - \frac{1}{p}\sum_{k=1}^p P_k(w_k^*(a);a) \leq \gamma (a-w^*)^2, \forall a\in \RB \nonumber
\end{align}
where $\gamma = \frac{1}{p}\sum_{k=1}^p \frac{(m-m_k)^2}{m_k}$.
\end{lemma}
\begin{proof}
For convenience, we define three sets
\begin{align}
	& K_1(a) = \{k|(m-m_k)y + b < -1\} \nonumber \\
	& K_2(a) = \{k|(m-m_k)y + b \in [-1,1]\} \nonumber \\
	& K_3(a) = \{k|(m-m_k)y + b > 1\} \nonumber
\end{align}
Then it is easy to find that
\begin{itemize}
	\item if $k\in K_1(a)$, then $w_k^*(a) = -\frac{(m-m_k)a+b+1}{m_k}$, $P_k(w_k^*(a);a) = -\frac{[(m-m_k)a+b+1]^2}{2m_k} + c_k$;
	\item if $k\in K_2(a)$, then $w_k^*(a) = 0$, $P_k(w_k^*(a);a) = c_k$;
	\item if $k\in K_3(a)$, then $w_k^*(a) = -\frac{(m-m_k)a+b-1}{m_k}$, $P_k(w_k^*(a);a) = -\frac{[(m-m_k)a+b-1]^2}{2m_k} + c_k$;
\end{itemize}

Now we calculate the Local-Global Gap.

Firstly we consider the case that $b<-1$, we have $w^* = -\frac{b+1}{m}$, $P(w^*) = -\frac{(b+1)^2}{2m}+c$, and
\begin{align}
	l_\pi(y) = & -\frac{(b+1)^2}{2m} - \frac{1}{p}[\sum_{k\in K_1(a)}-\frac{[(m-m_k)a+b+1]^2}{2m_k} \nonumber \\
			   & +\sum_{k\in K_3(a)}-\frac{[(m-m_k)a+b-1]^2}{2m_k}] \nonumber \\
	         = & -\frac{mw^{*2}}{2} + \frac{1}{p}[\sum_{k\in K_1(a)}\frac{[(m-m_k)a-mw^*]^2}{2m_k} \nonumber \\
	           & +\sum_{k\in K_3(a)}\frac{[(m-m_k)a-mw^*-2]^2}{2m_k}] \nonumber
\end{align}
For $k\in K_3(a)$, we have
\begin{align}
	     & [(m-m_k)a-mw^*-2]^2 \nonumber \\
	=    & [(m-m_k)a-mw^*]^2 - 4[(m-m_k)a-mw^*] + 4 \nonumber \\
	=    & [(m-m_k)a-mw^*]^2 - 4[(m-m_k)a+b+1] + 4 \nonumber \\
	\leq & [(m-m_k)a-mw^*]^2 \nonumber
\end{align}
Then we have
\begin{align}
	l_\pi(a) \leq & -\frac{mw^{*2}}{2} + \frac{1}{p}[\sum_{k\in K_1(a)}\frac{[(m-m_k)a-mw^*]^2}{2m_k} \nonumber \\
	              & +\sum_{k\in K_3(a)}\frac{[(m-m_k)a-mw^*]^2}{2m_k}] \nonumber \\
	         \leq & -\frac{mw^{*2}}{2} + \frac{1}{p}\sum_{k=1}^p\frac{[(m-m_k)a-mw^*]^2}{2m_k} \nonumber \\
	         =    &  \frac{1}{p}\sum_{k=1}^p \frac{(m-m_k)^2}{m_k}(a-w^*)^2 \nonumber
\end{align}

Secondly we consider the case that $b\in [-1,1]$, we have $w^*=0$, $P(w^*) = c$, and
\begin{align}
	l_\pi(a) = &  \frac{1}{p}[\sum_{k\in K_1(a)}\frac{[(m-m_k)a+b+1]^2}{2m_k} \nonumber \\
	           &+ \sum_{k\in K_3(a)}\frac{[(m-m_k)a+b-1]^2}{2m_k}] \nonumber
\end{align}

For $k\in K_1(a)$, we have $(m-m_k)a < -(b+1) \leq 0$, which means that $[(m-m_k)a]^2 > (b+1)^2$. Then we have
\begin{align}
	     &[(m-m_k)a+b+1]^2 \nonumber \\
	=    &[(m-m_k)a]^2 + (b+1)^2 + 2(b+1)(m-m_k)a \nonumber \\
	\leq & 2[(m-m_k)a]^2 \nonumber
\end{align}
For $k\in K_3(a)$, with the similar reason, we have
\begin{align}
	     &[(m-m_k)a+b-1]^2 \nonumber \\
	=    &[(m-m_k)a]^2 + (b-1)^2 + 2(b-1)(m-m_k)a \nonumber \\
	\leq & 2[(m-m_k)a]^2 \nonumber
\end{align}

Then we have
\begin{align}
	l_\pi(a) \leq &\frac{1}{p}[\sum_{k\in K_1(a)}\frac{2[(m-m_k)a]^2}{2m_k} + \sum_{k\in K_3(a)}\frac{2[(m-m_k)a]^2}{2m_k}] \nonumber \\
	             \leq &\frac{1}{p}\sum_{k=1}^p\frac{2[(m-m_k)a]^2}{2m_k} \nonumber \\
	             = & \frac{1}{p}\sum_{k=1}^p\frac{(m-m_k)^2}{m_k}(a-w^*)^2 \nonumber
\end{align}

For the case $b<-1$, it is the same as that $b<-1$. We have $l_\pi(a) \leq \frac{1}{p}\sum_{k=1}^p \frac{(m-m_k)^2}{m_k}(a-w^*)^2$.
\end{proof}

So that any $\pi$ is a qualified partition in this case. In high dimension space with diagonal positive definite matrices, we can also get similar result:

\begin{lemma}\label{lemma: quadratic function high}
Let $\phi_k(\w) = \frac{1}{2} \w^T\A_k\w + \b_k^T \w + c_k$, $F(\w) = \frac{1}{2}\w^T\A\w + \b^T \w + c$, $R(\w) = \|\w\|_1$, where $\A_k, \A \in \RB^{d\times d}$ are diagonal positive definite matrices,
$\x, \b, \b_k \in \RB^d, d>1$ and $c_k, c\in \RB$ and $[\phi_1(\cdot),\phi_2(\cdot),\ldots,\phi_p(\cdot)]\in A(P)$. With $\gamma = \mathop{\max}_{i=1,\ldots,d}\frac{1}{p}\sum_{k=1}^p \frac{(A(i,i)-A_k(i,i))^2}{A_k(i,i)}$, where $A(i,i)$, $A_k(i,i)$ are the $i^{th}$ diagonal elements
of $\A, \A_k$ respectively, then we have
\begin{align}
P(\w^*) - \frac{1}{p}\sum_{k=1}^p P_k(\w_k^*(\a);\a) \leq \gamma \|\a-\w^*\|^2, \forall \a\in \RB^d. \nonumber
\end{align}
\end{lemma}

One can find that although the local objective function in (\ref{eq:lgp_quadratic}) is non-smooth, the corresponding local-global gap is differentiable at $w^*$. This is not surprise. In fact, for any partition $\pi = [\phi_1(\cdot),\ldots,\phi_p(\cdot)]$, according to the dual form in Lemma \ref{lem:dual} and the duality of strong convexity~\cite{Kakade2009OnTD}, we obtain the following lemma:

\begin{lemma}\label{lem:gqf}
Let $\phi_k(\w) = \frac{1}{2} \w^T\A_k\w + \b_k^T \w + c_k$, $F(\w) = \frac{1}{2}\w^T\A\w + \b^T \w + c$, $R(\w) = \|\w\|_1$, where $\A_k, \A \in \RB^{d\times d}$ are positive definite matrices, $\w, \b, \b_k \in \RB^d, d>1$ , $c_k, c\in \RB$ and $[\phi_1(\cdot),\phi_2(\cdot),\ldots,\phi_p(\cdot)]\in A(P)$. then there must exist $\gamma$ such that
\begin{align}
P(\w^*) - \frac{1}{p}\sum_{k=1}^p P_k(\w_k^*(\a);\a) \leq \gamma \|\a-\w^*\|^2, \forall \a\in \RB^d. \nonumber
\end{align}
\end{lemma}
\begin{proof}
Since $\phi_k(\w) + R(\w)$ is strongly convex, the corresponding conjugate function is smooth. Noting that $G_k(\a)$ is an affine function, we have $H_k^*(-G_k(\a))$ is smooth as well, which means $l_\pi(\a)$ is smooth w.r.t $\a$. Then according to Lemma \ref{lem:dual}, we have $l_\pi(\w^*) = 0, \nabla l_\pi(\w^*) = \0$. With Definition \ref{def:smooth}, there must exist $\gamma >0$ such that $l_\pi(\a) \leq \gamma\|\a - \w^*\|^2, \forall \a$.
\end{proof}

\subsubsection{Extension to general $\phi_k(\cdot)$}

Below, we will consider the general case. To evaluate $l_\pi(\a)$ for general $\phi_k(\cdot)$, we consider the Taylor expansion of $\phi_k(\cdot)$ using its smoothness:for any fixed $\w_0$, let
\begin{align}
	  & \tilde{\phi}_{k,\w_0}(\w) \nonumber \\
	= & \phi_k(\w_0) + \nabla \phi_k(\w_0)^T(\w-\w_0) + \frac{\mu_k}{2}\|\w-\w_0\|^2 \label{eq:taylorlow}
\end{align}
\begin{align}
	  & \hat{\phi}_{k,\w_0}(\w) \nonumber \\
	= & \phi_k(\w_0) + \nabla \phi_k(\w_0)^T(\w-\w_0) + \frac{L_k}{2}\|\w-\w_0\|^2
	    \label{eq:taylorup}
\end{align}
which are two quadratic functions satisfy the condition of Lemma \ref{lemma: quadratic function high}. By the smooth and strong convex property of $\phi_k(\cdot)$, we have $\tilde{\phi}_{k,\w_0}(\w) \leq \phi_k(\w) \leq \hat{\phi}_{k,\w_0}(\w), \forall \w$. Now we proof the theorem:

Let $R(\w) = \|\w\|_1$. $\forall \pi \in A(P)$, there exists constant $\gamma< \infty$ such that $l_\pi(\a) \leq \gamma \|\a - \w^*\|^2, \forall \a$.

\begin{proof}
We define
\begin{align}
    \tilde{\phi}_{k,\w_0}(\w;\a) &= \tilde{\phi}_{k,\w_0}(\w) + (\nabla \tilde{\phi}_{\w_0}(\a) - \nabla \tilde{\phi}_{k,\w_0}(\a))^T\w \nonumber \\
    \tilde{P}_{k,\w_0}(\w;\a) &= \tilde{\phi}_{k,\w_0}(\w;\a) + R(\w) \nonumber
\end{align}

where $\tilde{\phi}_{\w_0}(\a) = \frac{1}{p}\sum_{k=1}^p \tilde{\phi}_{k,\w_0}(\a)$. It is easy to find that taking $\a = \w_0$,
\begin{align*}
     &\tilde{P}_{k,\w_0}(\w;\w_0) \\
=    & \tilde{\phi}_{k,\w_0}(\w) + (\nabla \tilde{\phi}_{\w_0}(\w_0) - \nabla \tilde{\phi}_{k,\w_0}(\w_0))^T\w + R(\w) \\
=    & \tilde{\phi}_{k,\w_0}(\w) + (\nabla F(\w_0) - \nabla \phi_{k}(\w_0))^T\w + R(\w) \\
\leq & \phi_k(\w) + G_k(\w_0)^T\w + R(\w) \\
=    & P_{k}(\w;\w_0)
\end{align*}
where $G_k(\cdot), P_{k}(\cdot;\cdot)$ is defined in (\ref{eq:pk}). It implies that
\begin{align}\label{eq:the_1.1}
  -\min_{\w} P_{k}(\w;\w_0) \leq -\min_{\w} \tilde{P}_{k,\w_0}(\w;\w_0)
\end{align}

Since $\tilde{\phi}_{k,\w_0}(\w)$, defined in (\ref{eq:taylorlow}), is quadratic function w.r.t $\w$ that satisfied the condition in Lemma \ref{lemma: quadratic function high}, there must exist a constant $\gamma_1$ (note that the constant $\gamma_1$ is independent on $\w_0$) such that $\forall \a$,
\begin{align}
	\mathop{\min}_{\w} \tilde{P}_{\w_0}(\w) - \frac{1}{p}\sum_{k=1}^p\mathop{\min}_{\w} \tilde{P}_{k,\w_0}(\w;\a) \leq \gamma_1\|\a-\tilde{\w}^*(\w_0)\|^2,  \label{eq:the_1.2}
\end{align}
where $\tilde{P}_{\w_0}(\w) = \frac{1}{p}\sum_{k=1}^p \tilde{\phi}_{k,\w_0}(\w) + R(\w)$, $\tilde{\w}^*(\w_0) = \mathop{\arg\min}_\w \tilde{P}_{\w_0}(\w)$.

Furthermore, by defining $\hat{P}_{\w_0}(\w) = \frac{1}{p}\sum_{k=1}^p \hat{\phi}_{k,\w_0}(\w) + R(\w)$, where $\hat{\phi}_{k,\w_0}(\w)$ is defined in (\ref{eq:taylorup}), we get that
\begin{align}\label{eq:the_1.3}
	P(\w) \leq \hat{P}_{\w_0}(\w), \forall \w
\end{align}

Taking $\a = \w_0$ in (\ref{eq:the_1.2}), $\bar{L} = \frac{1}{p}\sum_{i=k}^p L_k, \bar{\mu} = \frac{1}{p}\sum_{i=k}^p \mu_k$,  we obtain
\begin{align}
	     & P(\w^*) - \frac{1}{p}\sum_{k=1}^p\mathop{\min}_{\w} P_k(\w;\w_0)  \nonumber \\
	=    & P(\w^*) - \tilde{P}_{\w_0}(\tilde{\w}^*(\w_0)) + \tilde{P}_{\w_0}(\tilde{\w}^*(\w_0)) - \frac{1}{p}\sum_{k=1}^p\mathop{\min}_{\w} P_k(\w;\w_0) \nonumber \\
	\overset{(\ref{eq:the_1.1})}{\leq} & P(\w^*) - \tilde{P}_{\w_0}(\tilde{\w}^*(\w_0)) +\tilde{P}_{\w_0}(\tilde{\w}^*(\w_0)) - \frac{1}{p}\sum_{k=1}^p\mathop{\min}_{\w} \tilde{P}_{k,\w_0}(\w;\w_0) \nonumber \\
	\overset{(\ref{eq:the_1.2})}{\leq} & P(\w^*) - \tilde{P}_{\w_0}(\tilde{\w}^*(\w_0)) + \gamma_1\|\w_0-\tilde{\w}^*(\w_0)\|^2 \nonumber \\
	\leq & P(\tilde{\w}^*(\w_0)) - \tilde{P}_{\w_0}(\tilde{\w}^*(\w_0)) + \gamma_1\|\w_0-\tilde{\w}^*(\w_0)\|^2 \nonumber \\
	\overset{(\ref{eq:the_1.3})}{\leq} & \hat{P}_{\w_0}(\tilde{\w}^*(\w_0)) - \tilde{P}_{\w_0}(\tilde{\w}^*(\w_0)) + \gamma_1\|\w_0-\tilde{\w}^*(\w_0)\|^2 \nonumber \\
	= & (\frac{\bar{L}-\bar{\mu}}{2}+\gamma_1)\|\w_0-\tilde{\w}^*(\w_0)\|^2 \nonumber \\
	\leq & (\bar{L}-\bar{\mu}+2\gamma_1)(\|\w_0 - \w^*\|^2+\|\w^*-\tilde{\w}^*(\w_0)\|^2) \nonumber \\
	= & (\bar{L}-\bar{\mu}+2\gamma_1)(\|\w_0 - \w^*\|^2+\|\tilde{\w}^*(\w^*)-\tilde{\w}^*(\w_0)\|^2) \nonumber
\end{align}
where the last inequality using the fact that $\tilde{\w}^*(\w^*) = \w^*$. Using Lemma \ref{lemma: Lipschitz continue}, which clarifies the Lipschitz continuity of $\tilde{\w}^*(\cdot)$, we get that there must exist some constant $\gamma$ such that
\begin{align}
P(\w^*) - \frac{1}{p}\sum_{k=1}^p\mathop{\min}_{\w} P_k(\w;\w_0) \leq \gamma\|\w_0-\w^*\|^2
\end{align}
\end{proof}

We denote $\tilde{\w}^*(\w_0) = \mathop{\arg\min}_{\w} \tilde{P}_{\w_0}(\w)$, $\bar{L} = \frac{1}{p}\sum_{i=k}^p L_k, \bar{\mu} = \frac{1}{p}\sum_{i=k}^p \mu_k$. We have the following lemma:

\begin{lemma}\label{lemma: Lipschitz continue}
$\tilde{\w}^*(\w)$ is Lipschitz continue w.r.t $\w$.
\end{lemma}

\begin{proof}
We use $w^{(i)}$ to denote the $i^{th}$ element of $\w$ and define $\psi^{i}(\w) = \bar{\mu} w^{(i)} - (\nabla F(\w))^{(i)}$. Since $\nabla F(\w)$ is Lipschitz continue, so as $\psi^{i}(\w)$ (assume it is
$\alpha$-Lipschitz continue). We define three sets $A^{i} = \{\w|\psi^{i}(\w) > 1\}, B^{i} = \{\w|\psi^{i}(\w) \in [-1,1]\}, C^{i} = \{\w|\psi^{i}(\w) < -1\}$. According to the definition of $\tilde{P}_{\w}(\x)$, which is a quadratic function, we have
\begin{align}
(\tilde{\w}^*(\w))^{(i)} = & \frac{1}{\bar{\mu}}(\psi^{i}(\w) -1), \mbox{if } \w\in A^{i} \nonumber \\
(\tilde{\w}^*(\w))^{(i)} = & 0, \mbox{if } \w\in B^{i} \nonumber \\
(\tilde{\w}^*(\w))^{(i)} = & \frac{1}{\bar{\mu}}(\psi^{i}(\w) +1), \mbox{if } \w\in C^{i} \nonumber
\end{align}
First it is easy to note that $(\tilde{\w}^*(\w))^{(i)}$ is Lipschitz continue w.r.t $\w$ on $A^{i}, B^{i}, C^{i}$ respectively and is continue w.r.t $\w$ on the whole domain. Second, we take three points
$\w_1\in A^{i}, \w_2\in B^{i}, \w_3\in C^{i}$. Since $\psi^{i}(\w)$ is continue w.r.t $\w$ on the whole domain, there must be $\theta_{12}, \theta_{13},  \theta_{23}\in [0,1]$ such that
$\psi^{i}(\theta_{12}\w_1 + (1-\theta_{12})\w_2) = 1, \psi^{i}(\theta_{13}\w_1 + (1-\theta_{13})\w_3) \in B^i, \psi^{i}(\theta_{23}\w_2 + (1-\theta_{23})\w_3) = -1$. It implies that
\begin{align}
     &|(\tilde{\w}^*(\w_1))^{(i)} - (\tilde{\w}^*(\w_2))^{(i)}| \nonumber \\
=    &\frac{1}{\bar{\mu}}(\psi^{i}(\w_1) -1) \nonumber \\
=    &\frac{1}{\bar{\mu}}(\psi^{i}(\w_1) - \psi^{i}(\theta_{12}\w_1 + (1-\theta_{12})\w_2)) \nonumber \\
\leq &\frac{\alpha(1-\theta_{12})}{\bar{\mu}}\|\w_1-\w_2\| \nonumber \\
\leq &\frac{\alpha}{\bar{\mu}}\|\w_1-\w_2\| \nonumber
\end{align}
Similarly we can find that $|(\tilde{\w}^*(\w_1))^{(i)} - (\tilde{\w}^*(\w_3))^{(i)}| \leq \frac{\alpha}{\bar{\mu}}\|\w_1-\w_3\|$, $|(\tilde{\w}^*(\w_2))^{(i)} - (\tilde{\w}^*(\w_3))^{(i)}| \leq \frac{\alpha}{\bar{\mu}}\|\w_2-\w_3\|$.
So $(\tilde{\w}^*(\w))^{(i)}$ is Lipschitz continue. It implies that $\tilde{\w}^*(\w)$ is Lipschitz continue w.r.t $\w$.
\end{proof}

The result in Theorem~\ref{theorem:L1 qualified pi} can be easily extended to smooth regularization:
\begin{theorem}\label{theorem:smooth qualified pi}
Let $R(\cdot)$ be smooth. $\forall \pi \in A(P)$, there exists constant $\gamma< \infty$ such that
\begin{align*}
	l_\pi(\a) \leq \gamma \|\a - \w^*\|, \forall \a.
\end{align*}
\end{theorem}
\begin{proof}
Since $R(\w)$ is smooth, so as $P_k(\w;\a)$ w.r.t $\w$, then there exist some constant $L{'}$ (independent on $\a$) such that
\begin{align*}
l_\pi(\a) =    &\frac{1}{p}\sum_{k=1}^p P_k(\w^*;\a) - P_k(\w_k^*(\a);\a) \nonumber \\
          \leq &\frac{1}{p}\sum_{k=1}^p \frac{L{'}}{2}\|\w^*-\w_k^*(\a)\|^2
\end{align*}
Here we use the fact $P(\w^*)=\frac{1}{p}\sum_{k=1}^p P_k(\w^*;\a)$ and $\nabla P_k(\w_k^*(\a);\a)=0$. On the other hand, we have
\begin{align*}
&\nabla \phi_k(\w_k^*(\a)) + \nabla F(\a) - \nabla \phi_k(\a) + \nabla R(\w_k^*(\a)) = 0\\
&\nabla F(\w^*) + \nabla R(\w^*) = 0
\end{align*}
which means that
\begin{align*}
  &\|\nabla \phi_k(\w_k^*(\a))+ \nabla R(\w_k^*(\a))-\nabla \phi_k(\w^*)-\nabla R(\w^*)\| \\
= &\|\nabla \phi_k(\a)-\nabla \phi_k(\w^*)-\nabla R(\w^*) - \nabla F(\a)\| \\
= &\|\nabla \phi_k(\a)-\nabla \phi_k(\w^*)+\nabla F(\w^*) - \nabla F(\a)\|
\end{align*}
By strong convexity of $\phi_k(\w)+R(\w)$, smoothness of $\phi_k(\w)$, $F(\w)$, we have
\begin{align*}
\|\w^*-\w_k^*(\a)\| \leq \frac{L_k + L}{\mu_k}\|\a-\w^*\|
\end{align*}
So there must be some constant $\gamma$ such that $l_\pi(\a)\leq \gamma \|\a-\w^*\|^2$.
\end{proof}

\subsection{Condition of continuity of $\gamma(\pi;0)$}
Let $\pi^* = [F(\cdot),\ldots, F(\cdot)]$, where $F(\cdot)=\frac{1}{n} \sum_{i=1}^{n} f_i(\cdot)$ is defined in~(\ref{eq:obj}). We can find that $\pi^*$ is the best partition since $l_{\pi^*}(\a) = 0, \forall \a$, which implies
\begin{align}
\gamma(\pi^*;0) = 0 \nonumber
\end{align}
In the following content, we will prove that $\gamma(\pi;\epsilon) \rightarrow 0$ as $\pi$ approaches $\pi^*$. We first define the distance between two partitions as follows:
\begin{definition}\label{def:distance_partition}
Let $\pi_1=[\phi_1(\cdot),\ldots,\phi_p(\cdot)]$ and $\pi_2=[\psi_1(\cdot),\ldots,\psi_p(\cdot)]$ be two partitions w.r.t. $P(\cdot)$, the distance between $\pi_1$ and $\pi_2$ is
\begin{align*}
  &d(\pi_1,\pi_2) = \mathop{\max}_{i=1,\ldots,p} \{\mathop{\sup}_{\w}|\phi_i(\w)-\psi_i(\w)|,\mathop{\sup}_{\w}\|\nabla \phi_i(\w)- \nabla \psi_i(\w)\|\}.
\end{align*}
\end{definition}
We can prove that it is a reasonable definition of distance. For convenience, we assume the domain of these $\phi_k(\cdot)$ and $\psi_k(\cdot)$ is closed and bounded, which is denoted as $\mathcal{W}=\{\w|\|\w\|\leq B\}$ and $\w^*\in \mathcal{W}$. Hence, $d(\pi_1,\pi_2) < \infty$.

With the definition of distance between two partitions, we have
\begin{lemma}\label{lem:uniformly converge}
Let $\pi=[\phi_1(\cdot),\ldots,\phi_p(\cdot)]$ be a partition w.r.t. $P(\cdot)$. $l_\pi(\a)$ uniformly converges to $l_{\pi^*}(\a) = 0$ as $d(\pi,\pi^*) \rightarrow 0$.
\end{lemma}

Besides the uniform convergence of $l_\pi(\a)$ w.r.t. $\pi$, we also prefer the uniform convergence of $\gamma(\pi;\epsilon)$ w.r.t. $\pi$ since it reflects the goodness of partition $\pi$. According to the definition of $\gamma(\pi;\epsilon)$, it easy to note that $\gamma(\pi;0) \geq \gamma(\pi;\epsilon), \forall \epsilon \geq 0$. Below, we would like to explore the necessary and sufficient condition of continuity of $\gamma(\pi;0)$ at $\pi^*$. For convenience, we only consider one dimension space, which means $d=1$(In high dimension space, similar result could be got). We assume $l_\pi(a) \in C^3(\mathcal{W},\RB)$ (note that $\mathcal{W}$ is compact).

\begin{lemma}\label{lem:necessary1}
If $\lim_{\pi \to \pi^*}\gamma(\pi;0)=0$, then $l_\pi^{''}(w^*)$ converge to $l_{\pi^*}^{''}(w^*) = 0$.
\end{lemma}
\begin{proof}
If it is wrong, there must exist partitions $\{\pi_n\}$ and constant $a$ such that $l_{\pi_n}^{''}(w^*) > a>0$ (or $<a<0$), and $d(\pi_n,\pi_*) \rightarrow 0$ as $n \rightarrow \infty$. Since $l_{\pi_n}(w^*) = l_{\pi_n}^{'}(w^*)=0$, then for the series $\{y_n\}$ with $y_n = w^* + \frac{1}{n\sup_x{|l_{\pi_n}^{'''}(x)}|}$, by Taylor theory, we obtain:
\begin{align*}
l_{\pi_n}(y_n) = \frac{1}{2}l_{\pi_n}^{''}(w^*)(y_n-w^*)^2 + \frac{1}{6}l_{\pi_n}^{'''}(\zeta_n)(y_n-w^*)^3
\end{align*}
which implies
\begin{align*}
\gamma(\pi_n;0) \geq & \frac{1}{2}l_{\pi_n}^{''}(w^*) + \frac{1}{6}l_{\pi_n}^{'''}(\zeta_n)(y_n-w^*) \nonumber \\
                >    & \frac{a}{2} + \frac{l_{\pi_n}^{'''}(\zeta_n)}{6n\sup_x{|l_{\pi_n}^{'''}(x)}|}
\end{align*}
Let $n\rightarrow \infty$, it conflicts with the condition.
\end{proof}

For the necessary and sufficient condition, we have the following result:
\begin{lemma}\label{lemma:continue wrt pi}
$\lim_{\pi \to \pi^*}\gamma(\pi;0)=0$ if and only if there exist some $\delta$ such that $l_\pi^{''}(y)$ uniformly converge to $l_{\pi^*}^{''}(y) = 0$ on $\{y||y-w^*|<\delta\}$ w.r.t $\pi$.
\end{lemma}
\begin{proof}
First we proof the sufficiency.
If it is wrong, there must exist some qualified partitions $\pi_n$, $y_n$ and constant $a$ such that $l_{\pi_n}(y_n) > a(y_n-w^*)^2>0$, and $d(\pi_n,\pi_*) \rightarrow 0$ as $n \rightarrow \infty$.

Since $l_\pi^{''}(y)$ uniformly converge to $0$, one direct result is that $\{l_\pi^{''}(y)\}$ is equicontinuous, which means $\forall \epsilon > 0, \exists \delta > 0$ such that
\begin{align*}
|l_\pi^{''}(y_1) - l_\pi^{''}(y_2)| < \epsilon, \forall |y_1-y_2|<\delta, \pi.
\end{align*}
By taking $y_2 = w^*, \epsilon = \frac{a}{4}$, since $\lim_{\pi \rightarrow \pi^*} l_{\pi}^{''}(w^*) = 0$, we obtain $\exists 0<\delta_0<\delta$
\begin{align*}
|l_{\pi}^{''}(y)| < \frac{a}{4} + |l_\pi^{''}(w^*)| < \frac{a}{2} , \forall |y-w^*|<\delta_0, d(\pi,\pi^*)<\delta_0
\end{align*}
which implies $\forall |y-w^*|<\delta_0, d(\pi,\pi^*)<\delta_0$,
\begin{align*}
  &l_{\pi}(y) = \int_{w^*}^y l_{\pi}^{'}(t)dt \in (\frac{-a(y-w^*)^2}{4}, \frac{a(y-w^*)^2}{4})
\end{align*}
It implies that $|y_n-w^*|\geq\delta_0$. However, according to Lemma \ref{lem:uniformly converge}, we have
\begin{align*}
  \frac{l_{\pi_n}(y_n)}{(y_n-w^*)^2} \leq \mathop{\sup}_{|y-w^*|\geq\delta_0} \frac{l_{\pi_n}(y)}{(y-w^*)^2} \leq \mathop{\sup}_{|y-w^*|\geq\delta_0} \frac{l_{\pi_n}(y)}{\delta_0^2} \rightarrow 0
\end{align*}

which conflicts with the assumption at the beginning of proof.

For the necessity, if it is wrong, there must exit $\pi_n, y_n$ and constant $a$ such that $|l_{\pi_n}^{''}(y_n)| > a > 0$, and $d(\pi_n,\pi_*) \rightarrow 0$, $|y_n - w^*| \rightarrow 0$ as $n \rightarrow \infty$.

According to Lemma \ref{lem:necessary1}, we have $\lim_{n\rightarrow \infty} l_{\pi_n}^{''}(w^*) = 0$. On the other hand, by Taylor theory, we obtain
\begin{align*}
l_{\pi_n}^{''}(w^*) = l_{\pi_n}^{''}(y_n) + \mathcal{O}(y_n - w^*) \rightarrow 0 (n\rightarrow \infty)
\end{align*}
which implies that $\exists \delta_0$ such that $|y_n - w^*| > \delta_0$ with sufficient large $n$, which makes the confliction.
\end{proof}

Above all, we get the sufficient and necessary condition of $\lim_{\pi \to \pi^*}\gamma(\pi;0)=0$. As an example in Lemma \ref{lemma: quadratic function}, it satisfies the condition that $l^{''}_\pi(y)$ uniform converge to 0.

\subsection{Proof of Lemma \ref{lemma: Good partition}}
Assume $\pi=[F_1(\cdot),\ldots,F_p(\cdot)]$ is a partition w.r.t. $P(\cdot)$, where $F_k(\w) = \frac{1}{|D_k|}\sum_{i \in D_k}f_i(\w)$ is the local loss function with bounded domain, each $f_i(\cdot)$ is Lipschitz continuous and sampled from some unknown distribution $\PB$.
If we assign these $\{f_i(\cdot)\}$ uniformly to each worker, then with high probability, $\gamma(\pi; \epsilon) \leq  \frac{1}{p}\sum_{k=1}^p \mathcal{O}(1/(\epsilon\sqrt{|D_k|}))$. Moreover, if $l_\pi(\a)$ is convex w.r.t. $\a$, then $\gamma(\pi; \epsilon) \leq  \frac{1}{p}\sum_{k=1}^p \mathcal{O}(1/\sqrt{\epsilon|D_k|})$. Here we hid the $\log$ term and dimensionality $d$.

\begin{proof}
For convenience, We define
\begin{align}
\tilde{F}(\w) = & \E_{\xi\sim\P} [f_\xi(\w)] \nonumber \\
\tilde{P}(\w) = & \tilde{F}(\w) + R(\w) \nonumber
\end{align}
Then both $F(\w)$ and $F_k(\w)$ are the empirical estimation of $\tilde{F}(\w)$. Let the domain of $f_\xi(\cdot)$ be $\mathcal{W} = \{\w|\|\w\|\leq B\}$. By the uniform convergence (theorem 5 in
(\cite{DBLP:conf/colt/Shalev-ShwartzSSS09})), we obtain that with high probability,
\begin{align}
|P_k(\w;\a) - \tilde{P}(\w)| \leq & |F_k(\w) -\tilde{F}(\w)| + \|\w\|\|\nabla F(\a) - \nabla F_k(\a)\| \nonumber \\
                             \leq & \mathcal{O}(\sqrt{\frac{d}{|D_k|}}) + B\|\nabla \tilde{F}(\a) - \nabla F(\a)\| + B\|\nabla \tilde{F}(\a) - \nabla F_k(\a)\| \nonumber \\
                             \leq & \mathcal{O}(\sqrt{\frac{d}{|D_k|}}) + \mathcal{O}(\frac{dB}{\sqrt{n}}) + \mathcal{O}(\frac{dB}{\sqrt{|D_k|}})\nonumber \\
                             =    & \mathcal{O}(\frac{1}{\sqrt{|D_k|}}) \nonumber
\end{align}
Here we ignore the $\log$ term.
Similarly, we have $|P(\w) - \tilde{P}(\w)| \leq \mathcal{O}(\frac{1}{\sqrt{n}})$, which means that
\begin{align}
|P_k(\w;\a) - P(\w)| \leq \mathcal{O}(\frac{1}{\sqrt{|D_k|}}) \nonumber
\end{align}
Then we have
\begin{align}
l_\pi(\a) = & \frac{1}{p}\sum_{k=1}^p P(\w^*) - P_k(\w_k^*(\a);\a) \nonumber \\
                            \leq &\frac{1}{p}\sum_{k=1}^p \mathop{\max}_{\w} |P(\w) - P_k(\w;\a)| \nonumber \\
                            \leq &\frac{1}{p}\sum_{k=1}^p \mathcal{O}(\frac{1}{\sqrt{|D_k|}}) \nonumber
\end{align}
We can note that the right term is independent on $\a$. If $|D_k|$ is large enough, the Local-Global Gap would be small. And
\begin{align}
\gamma(\pi; \epsilon) = \mathop{\sup}_{\|\a-\w^*\|^2 \geq \epsilon}\frac{l_\pi(\a)}{\|\a-\w^*\|^2} \leq \frac{1}{p}\sum_{k=1}^p O(\frac{1}{\epsilon\sqrt{|D_k|}}) \nonumber
\end{align}
Moreover, if $l_\pi(\a)$ is convex and $\epsilon\leq \|\a-\w^*\|^2\leq 1$, then define $\a_1, \theta$ such that $\a - \w^* = \theta(\a_1-\w^*), \|\a_1-\w^*\|=1$. Then $\theta \geq \sqrt{\epsilon}$.
Since $l_\pi(\w^*) = 0$, then
\begin{align*}
	l_\pi(\a) = l_\pi(\theta\a_1 + (1-\theta)\w^*) \leq \theta l_\pi(\a_1)
\end{align*}
We get
\begin{align*}
\gamma(\pi; \epsilon) =    & \mathop{\sup}_{\|\a-\w^*\|^2 \geq \epsilon}\frac{l_\pi(\a)}{\|\a-\w^*\|^2} \leq \mathop{\sup}_{\|\a-\w^*\|^2 \geq \epsilon}\frac{l_{\pi}(\a_1)}{\theta} \nonumber \\
                      \leq & \frac{1}{p}\sum_{k=1}^p
O(\frac{1}{\sqrt{\epsilon|D_k|}})
\end{align*}
\end{proof}

\section{Convergence of pSCOPE}
\subsection{Proof of Theorem \ref{theorem:qualified segementation}}
\begin{proof}
Let $\u_{k,m+1} = \u_{k,m} - \eta\g_{k,m}$. $\forall \z$
\begin{align}
     & P_k(\u_{k,m+1};\w_t) \nonumber \\
\leq & F_k(\u_{k,m};\w_t) - \eta \nabla F_k(\u_{k,m};\w_t)^T(\g_{k,m}) + \frac{L\eta^2}{2}\|\g_{k,m}\|^2 + R(\u_{k,m} - \eta \g_{k,m}) \nonumber \\
\leq & F_k(\z;\w_t) - \nabla F_k(\u_{k,m};\w_t)^T(\z - \u_{k,m}) - \frac{\mu}{2}\|\z - \u_{k,m}\|^2 \nonumber \\
     & - \eta \nabla F_k(\u_{k,m};\w_t)^T(\g_{k,m}) + \frac{L\eta^2}{2}\|\g_{k,m}\|^2 \nonumber \\
     & + R(\z) - (\g_{k,m} - \v_{k,m})^T(\z - \u_{k,m} + \eta \g_{k,m}) \nonumber \\
=    & P_k(\z;\w_t) - \frac{\mu}{2}\|\z - \w_{k,m}\|^2 + (\v_{k,m}- \nabla F_k(\u_{k,m};\w_t))^T(\z - \u_{k,m} + \eta \g_{k,m}) \nonumber \\
     & - \g_{k,m}^T(\z - \u_{k,m}) - (\eta-\frac{L\eta^2}{2})\|\g_{k,m}\|^2 \nonumber
\end{align}
The first inequality uses the smooth property of $F_k(\cdot)$. The second inequlity uses the strongly convex property of $F_k(\cdot)$ and the fact that $\g_{k,m} - \v_{k,m} \in \partial R(\mbox{prox}_{R,\eta}(\u_{k,m} - \eta \v_{k,m})) = \partial R(\u_{k,m} - \eta \g_{k,m})$.

Taking $\z = \w^*, \eta \leq \frac{1}{L}$, then we have
\begin{align}
     & P_k(\u_{k,m+1};\w_t) \nonumber \\
\leq & P_k(\w^*;\w_t)  + (\v_{k,m}- \nabla F_k(\u_{k,m};\w_t))^T(\w^* - \u_{k,m+1}) \nonumber \\
     & - \frac{\mu}{2}\|\w^* - \u_{k,m}\|^2 - \g_{k,m}^T(\w^* - \u_{k,m}) - \frac{\eta}{2}\|\g_{k,m}\|^2 \label{eq:the_3.1}
\end{align}
According to the update rule, we have
\begin{align}
& \|\u_{k,m+1} - \w^*\|^2 \nonumber \\
=     & \|\u_{k,m} -\w^*- \eta \g_{k,m}\|^2 \nonumber \\
=     & \|\u_{k,m} -\w^*\|^2 + 2\eta\g_{k,m}^T(\w^*-\u_{k,m}) +\eta^2\|\g_{k,m}\|^2 \nonumber \\
\overset{(\ref{eq:the_3.1})}{\leq}  & \|\u_{k,m} -\w^*\|^2 + 2\eta(P_k(\w^*;\w_t) - P_k(\u_{k,m+1};\w_t)) - \mu\eta\|\u_{k,m}-\w^*\|^2 \nonumber \\
      & -2\eta(\v_{k,m} - \nabla F_k(\u_{k,m};\w_t))^T(\u_{k,m+1}-\w^*) \nonumber
\end{align}

Define $\hat{\u}_{k,m+1} = prox_{\eta R}(\u_{k,m} - \eta \nabla F_k(\u_{k,m};\w_t))$, $\Delta_{k,m} = \v_{k,m}- \nabla F_k(\u_{k,m};\w_t)$, then we have
\begin{align}
     &-2\eta(\v_{k,m}- \nabla F_k(\u_{k,m};\w_t))^T(\u_{k,m+1}-\w^*) \nonumber \\
=    & -2\eta\Delta_{k,m}^T(\u_{k,m+1}-\hat{\u}_{k,m+1}) - 2\eta\Delta_{k,m}^T(\hat{\u}_{k,m+1}-\w^*) \nonumber \\
\leq & 2\eta\|\Delta_{k,m}\|\|(\u_{k,m+1}-\hat{\u}_{k,m+1})\| - 2\eta\Delta_{k,m}^T(\hat{\u}_{k,m+1}-\w^*) \nonumber \\
\leq & 2\eta^2\|\Delta_{k,m}\|^2 - 2\eta\Delta_{k,m}^T(\hat{\u}_{k,m+1}-\w^*)
\end{align}
The last inequality uses the property $\forall \x,\y, \|\mbox{prox}_{R,\eta}(\x) - \mbox{prox}_{R,\eta}(\y)\| \leq \|\x - \y\|$.

We can note that for the expectation, we have
\begin{align}
\EB \Delta_{k,m} = \EB \v_{k,m} - \nabla F_k(\u_{k,m};\w_t) = 0 \nonumber
\end{align}

For the variance, we have
\begin{align}
     &\EB \|\Delta_{k,m}\|^2 \nonumber \\
=    & \EB \|\v_{k,m}- \nabla F_k(\u_{k,m};\w_t)\|^2 \nonumber \\
=    & \frac{1}{q}\sum_{i\in D_k}\| \nabla f_i(\u_{k,m}) - \nabla f_i(\w_t) - (\nabla F_k(\u_{k,m}) - \nabla F_k(\w_t)) \|^2 \nonumber \\
\leq & \frac{1}{q}\sum_{i\in D_k}\| \nabla f_i(\u_{k,m}) - \nabla f_i(\w_t) \|^2 \nonumber \\
\leq & L^2\|\u_{k,m}-\w_t\|^2 \nonumber \\
\leq & 2L^2\|\u_{k,m}-\w^*\|^2 + 2L^2\|\w_t-\w^*\|^2 \nonumber
\end{align}

Combining the above equations, we have
\begin{align}
     & \EB\|\u_{k,m+1} - \w^*\|^2 \nonumber \\
\leq & (1-\mu\eta + 2L^2\eta^2)\EB\|\u_{k,m} -\w^*\|^2 + 2L^2\eta^2\EB\|\w_t - \w^*\|^2 + 2\eta\EB(P_k(\w^*;\w_t) - P_k(\u_{k,m+1};\w_t)) \nonumber \\
\leq & (1-\mu\eta + 2L^2\eta^2)\EB\|\u_{k,m} -\w^*\|^2 + 2L^2\eta^2\EB\|\w_t - \w^*\|^2 + 2\eta(P_k(\w^*;\w_t) - P_k(\w^*(\w_t);\w_t)) \nonumber
\end{align}
where $\w^*(\w_t) = \mathop{\arg\min}_\w P_k(\w;\w_t)$. Let $\rho = 1-\mu\eta + 2L^2\eta^2$, then we have
\begin{align}
     & \EB\|\u_{k,M} - \w^*\|^2 \nonumber \\
\leq & (\rho^M+\frac{2L^2\eta^2}{1-\rho})\|\w_t - \w^*\|^2 + \frac{2\eta}{1-\rho}(P_k(\w^*;\w_t) - P_k(\w^*(\w_t);\w_t)) \nonumber
\end{align}

Summing up the equation from $k=1$ to $p$,
\begin{align}
     & \EB\|\w_{t+1} - \w^*\|^2 \nonumber \\
\leq & (\rho^M+\frac{2L^2\eta^2}{1-\rho})\|\w_t - \w^*\|^2 + \frac{2\eta}{(1-\rho)}l_\pi(\w_t) \nonumber \\
\leq & (\rho^M+\frac{2L^2\eta^2 + \eta \xi}{1-\rho})\|\w_t - \w^*\|^2 \nonumber \\
=    & ((1-\mu\eta + 2L^2\eta^2)^M+\frac{2L^2\eta + 2\xi}{\mu - 2L^2\eta})\|\w_t - \w^*\|^2 \nonumber
\end{align}
\end{proof}

\subsection{Proof of Corollary \ref{corollary:good segementation}}
\begin{proof}
Since $\gamma(\pi;\epsilon) \leq \frac{\mu}{8}$. Then
\begin{align}
\frac{2\zeta+2L^2\eta}{\mu - 2L^2\eta} \leq \frac{1}{2}
\end{align}
On the other hand,
\begin{align}
(1-\mu\eta + 2L^2\eta^2)^M = (1 - \frac{5}{72\kappa^2})^M \leq \frac{1}{4} \nonumber
\end{align}
So the computation complexity is $O((n/p+\kappa^2) \log(\frac{1}{\epsilon}))$
\end{proof}

\subsection{Proof of Corollary \ref{corollary:proxsvrg}}
\begin{proof}
Since $p = 1$, we have $l_\pi(\w_t) = 0, \gamma(\pi;0) = 0$. Then
\begin{align}
\frac{2L^2\eta}{\mu - 2L^2\eta} \leq \frac{1}{2}
\end{align}
and
\begin{align}
(1-\mu\eta + 2L^2\eta^2)^M = (1 - \frac{1}{9\kappa^2})^M \leq \frac{1}{4} \nonumber
\end{align}
So the computation complexity is $O((n+\kappa^2) \log(\frac{1}{\epsilon}))$
\end{proof}

\section{Handle High Dimensional Sparse Data}
\label{appendix:recovery rules}

The algorithm with recovery rule is as follow:

\begin{algorithm}[thb]
\caption{Efficient Learning for Sparse Data}
\label{alg:sparseupdate}
\small
\begin{algorithmic}[1]
\STATE  \textbf{Task of the $k${th} worker}:
\FOR{$t=0,1,2,...T-1$}
\STATE Wait until receiving $\w_t$ from master;
\STATE Let $\u_{k,0} = \w_t$, calculate $\z_k = \sum_{i\in D_k} f_i(\w_t)$ and send $\z_k$ to master;
\STATE Wait until receiving $\z$ from master;
\STATE Set $\r = (0,0,\ldots,0)\in \RB^d$;
\FOR{$m=0,1,2,...M-1$}
\STATE Randomly choose an instance index $s$ and denote $C_s = \{j|x_s^{(j)} \neq 0\}$;
\STATE $\forall j\in C_s$, recover $u^{(j)}_{k,m} $ according to the \emph{recovery rules} with $m_1 = r^{(j)}, m_2=m$;
\STATE Calculate the inner products $\x_s^T\u_{k,m}$ and $\x_s^T\w_t$;
\FOR{$j \in C_s$}
\STATE $v^{(j)}_{k,m} = h_s^{'}(\x_s^T\u_{k,m})x_s^{(j)} - h_s^{'}(\x_s^T\w_t)x_s^{(j)} + z^{(j)}$;
\STATE $u^{(j)}_{k,m} \leftarrow prox_{\lambda_2|u|, \eta}((1-\eta\lambda_1)u^{(j)}_{k,m} - \eta v^{(j)}_{k,m})$;
\STATE $r^{(j)} = m+1$;
\ENDFOR
\ENDFOR
\STATE  $\forall j\in [d]$, recover $u^{(j)}$ to get $\u_{k,M}$;
\STATE Send $\u_{k,M}$ to master
\ENDFOR
\end{algorithmic}
\end{algorithm}

First, we define two sequences $\{\alpha_q\}_{q=0,1,\ldots}, \{\beta_q\}_{q=1,2,\ldots}$ as: $\alpha_0 = 0$ and for $q = 1,2,\ldots$
\begin{align}
  \alpha_{q} = \frac{\beta_q}{(1-\lambda_1\eta)^q}, ~~\beta_q &= \sum_{i=1}^{q}(1-\lambda_1\eta)^{i-1}
\end{align}

\begin{lemma}
(\textbf{Recovery Rules}) For the coordinate $j$ and constants $m_1, m_2$, if $j \notin C_{i_{k,m}}$ for any $m\in [m_1, m_2-1]$, then the relation between $u_{k,m_1}^{(j)}$ and $u_{k,m_2}^{(j)}$ can be summarized as follow:
\begin{enumerate}
  \item $|z^{(j)}| < \lambda_2$:
        \begin{enumerate}
          \item $u_{k,m_1}^{(j)} > 0$. In this case, define $q_0$ satisfies $\alpha_{q_0}\eta(z^{(j)}+\lambda_2)\leq u_{k,m_1}^{(j)} < \alpha_{q_0+1}\eta(z^{(j)}+\lambda_2)$.

                If $m_2 - m_1 \leq q_0$, then
                \begin{align}
                	  &u_{k,m_2}^{(j)} \nonumber \\
                	= &(1-\lambda_1\eta)^{m_2-m_1}[u_{k,m_1}^{(j)} - \alpha_{m_2-m_1}\eta(z^{(j)}+\lambda_2)] \nonumber
                \end{align}
                If $m_2 - m_1 > q_0$, then
                      $$u_{k,m_2}^{(j)} = 0.$$
          \item $u_{k,m_1}^{(j)} = 0$.
                In this case, $u_{k,m_2}^{(j)} = 0$.
          \item $u_{k,m_1}^{(j)} < 0$.
                In this case, define $q_0$ satisfies $\alpha_{q_0+1}\eta(z^{(j)}-\lambda_2) < u_{k,m_1}^{(j)} \leq \alpha_{q_0}\eta(z^{(j)}-\lambda_2)$.

                If $m_2 - m_1 \leq q_0$, then
                      $$u_{k,m_2}^{(j)} = (1-\lambda_1\eta)^{q_0}[u_{k,m_1}^{(j)} - \alpha_{q_0}\eta(z^{(j)}-\lambda_2)].$$
                If $m_2 - m_1 > q_0$, then
                      $$u_{k,m_2}^{(j)} = 0.$$
        \end{enumerate}
  \item $z^{(j)} = -\lambda_2$.

        In this case, we have
        \begin{align*}
        u_{k,m_2}^{(j)} =
        \begin{cases}
          (1-\lambda_1\eta)^{(m_2-m_1)}u_{k,m_1}^{(j)}, & \mbox{if } u_{k,m_1}^{(j)} \geq 0  \\
          \mbox{the same as 1-(c)}, & \mbox{if } u_{k,m_1}^{(j)} < 0 \\
        \end{cases}
        \end{align*}
  \item $z^{(j)} = \lambda_2$.

        In this case we have
        \begin{align*}
        u_{k,m_2}^{(j)} =
        \begin{cases}
          (1-\lambda_1\eta)^{(m_2-m_1)}u_{k,m_1}^{(j)}, & \mbox{if } u_{k,m_1}^{(j)} \leq 0  \\
          \mbox{the same as 1-(a)}, & \mbox{if } u_{k,m_1}^{(j)} > 0 \\
        \end{cases}
        \end{align*}
  \item $z^{(j)} > \lambda_2$.
        \begin{enumerate}
          \item $u_{k,m_1}^{(j)} > 0$. In this case, define $q_0$ satisfies $\alpha_{q_0}\eta(z^{(j)}+\lambda_2)\leq u_{k,m_1}^{(j)} < \alpha_{q_0+1}\eta(z^{(j)}+\lambda_2)$.

                  If $m_2 - m_1 \leq q_0$, then
                        $$u_{k,m_2}^{(j)} = (1-\lambda_1\eta)^{q_0}[u_{k,m_1}^{(j)} - \alpha_{q_0}\eta(z^{(j)}+\lambda_2)].$$
                  If $m_2 - m_1 > q_0$ and $(1-\lambda_1\eta)u_{k,m_1 + q_0}^{(j)}\leq \eta(z^{(j)}-\lambda_2)$, then
                        \begin{align*}
                          u_{k,m_2}^{(j)} =& (1-\lambda_1\eta)^{m_2-m_1-q_0}u_{k,q_0}^{(j)} \\
                                           & - \eta(z^{(j)}-\lambda_2)\beta_{m_2-m_1-q_0}
                        \end{align*}
                  If $m_2 - m_1 = q_0+1$ and $(1-\lambda_1\eta)u_{k,m_1+q_0}^{(j)}> \eta(z^{(j)}-\lambda_2)$, then
                        \begin{align*}
                          u_{k,m_2}^{(j)} = 0
                        \end{align*}
                  If $m_2 - m_1 > q_0+1$ and $(1-\lambda_1\eta)u_{k,m_1+q_0}^{(j)}> \eta(z^{(j)}-\lambda_2)$, then
                        \begin{align*}
                          u_{k,m_2}^{(j)} = - \eta(z^{(j)}-\lambda_2)\beta_{m_2-m_1-q_0-1}
                        \end{align*}
          \item $u_{k,m_1}^{(j)} \leq 0$. In this case, we have
                  \begin{align*}
                    u_{k,m_2}^{(j)} =& (1-\lambda_1\eta)^{m_2-m_1}u_{k,m_1}^{(j)} \\
                                     & - \eta(z^{(j)}-\lambda_2)\beta_{m_2-m_1}
                  \end{align*}
        \end{enumerate}
  \item $z^{(j)} < -\lambda_2$
        \begin{enumerate}
          \item $u_{k,m_1}^{(j)} \geq 0$. In this case, we have
                  \begin{align*}
                    u_{k,m_2}^{(j)} =& (1-\lambda_1\eta)^{m_2-m_1}u_{k,m_1}^{(j)} \\
                                     & - \eta(z^{(j)}+\lambda_2)\beta_{m_2-m_1}
                  \end{align*}
          \item $u_{k,m_1}^{(j)} < 0$. In this case, define $q_0$ satisfies $\alpha_{q_0+1}\eta(z^{(j)}-\lambda_2) < u_{k,m_1}^{(j)} \leq \alpha_{q_0}\eta(z^{(j)}-\lambda_2)$.

                  If $m_2 - m_1 \leq q_0$, then
                      $$u_{k,m_2}^{(j)} = (1-\lambda_1\eta)^{q_0}[u_{k,m_1}^{(j)} - \alpha_{q_0}\eta(z^{(j)}-\lambda_2)].$$
                  If $m_2 - m_1 > q_0$ and $(1-\lambda_1\eta)u_{k,m_1 + q_0}^{(j)}\leq \eta(z^{(j)}+\lambda_2)$, then
                        \begin{align*}
                          u_{k,m_2}^{(j)} =& (1-\lambda_1\eta)^{m_2-m_1-q_0}u_{k,q_0}^{(j)} \\
                                           & - \eta(z^{(j)}+\lambda_2)\beta_{m_2-m_1-q_0}
                        \end{align*}
                  If $m_2 - m_1 = q_0+1$ and $(1-\lambda_1\eta)u_{k,m_1+q_0}^{(j)}> \eta(z^{(j)}+\lambda_2)$, then
                        \begin{align*}
                          u_{k,m_2}^{(j)} = 0
                        \end{align*}
                  If $m_2 - m_1 > q_0+1$ and $(1-\lambda_1\eta)u_{k,m_1+q_0}^{(j)}> \eta(z^{(j)}+\lambda_2)$, then
                        \begin{align*}
                          u_{k,m_2}^{(j)} = - \eta(z^{(j)}+\lambda_2)\beta_{m_2-m_1-q_0-1}
                        \end{align*}
        \end{enumerate}
\end{enumerate}
\end{lemma}
\begin{proof}
According to the definition of proximal mapping, we have
\begin{align}
  u_{k,m+1}^{(j)} =
  \begin{cases}
    d^{(j)}-\lambda_2\eta, & \mbox{if } d^{(j)} \geq \lambda_2\eta \\
    0, & \mbox{if } |d^{(j)}|<\lambda_2\eta \\
    d^{(j)}+\lambda_2\eta, & \mbox{if } d^{(j)} \leq -\lambda_2\eta
  \end{cases}
\end{align}
where $d(j) = (1-\lambda_1\eta)u_{k,m}^{(j)}-\eta z^{(j)}$.

First we consider the case $|z^{j}| < \lambda_2, u_{k,m_1}^{(j)} > 0$. According to the definition of $\{\alpha_q\}$, it is a monotonic increasing sequence. With the definition of $q_0$ that $\alpha_{q_0}\eta(z^{(j)}+\lambda_2)\leq u_{k,m_1}^{(j)} < \alpha_{q_0+1}\eta(z^{(j)}+\lambda_2)$. It is easy to verify that: for $q = 0,1,\ldots,\mathop{\min}\{q_0,m_2 - m_1\}$, we have
\begin{align}
u_{k,m_1+q}^{(j)} = (1-\lambda_1\eta)^q[u_{k,m_1}^{(j)} - \alpha_q\eta(z^{(j)}+\lambda_2)] \nonumber
\end{align}
So if $m_2 - m_1 \leq q_0$, then
\begin{align}
  & u_{k,m_2}^{(j)} \nonumber \\
= & (1-\lambda_1\eta)^{m_2-m_1}[u_{k,m_1}^{(j)} - \alpha_{m_2-m_1}\eta(z^{(j)}+\lambda_2)]
\end{align}
If $m_2 - m_1 > q_0$, according to the definition of $q_0$, we have
\begin{align}
    &(1-\lambda_1\eta)u_{k,m_1+q_0}^{(j)} - \eta z^{(j)} \nonumber \\
=   &(1-\lambda_1\eta)^{q_0+1}[u_{k,m_1}^{(j)} - \alpha_{q_0}\eta(z^{(j)}+\lambda_2)] - \eta z^{(j)}\nonumber \\
\in &(-\lambda_2\eta, \lambda_2\eta) \nonumber
\end{align}
Then we have $u_{k,m_1+q_0+1}^{(j)} = 0$. Since $0 - \eta z^{j} \in  (-\lambda_2\eta, \lambda_2\eta)$, then we have
\begin{align}
u_{k,m_2}^{(j)} = u_{k,m_2-1}^{(j)} = \cdots = u_{k,m_1+q_0+1}^{(j)} = 0
\end{align}

The next case is that $|z^{j}| < \lambda_2, u_{k,m_1}^{(j)} = 0$. Since $0 - \eta z^{j} \in  (-\lambda_2\eta, \lambda_2\eta)$, then we have
\begin{align}
u_{k,m_2}^{(j)} = u_{k,m_2-1}^{(j)} = \cdots = u_{k,m_1}^{(j)} = 0
\end{align}

The next case is that $|z^{j}| < \lambda_2, u_{k,m_1}^{(j)} < 0$. With the definition of $q_0$ that $\alpha_{q_0+1}\eta(z^{(j)}-\lambda_2)\leq u_{k,m_1}^{(j)} < \alpha_{q_0}\eta(z^{(j)}-\lambda_2)$. It is easy to verify that: for $q = 0,1,\ldots,\mathop{\min}\{q_0,m_2 - m_1\}$, we have
\begin{align}
u_{k,m_1+q}^{(j)} = (1-\lambda_1\eta)^q[u_{k,m_1}^{(j)} - \alpha_q\eta(z^{(j)}-\lambda_2)] \nonumber
\end{align}
So if $m_2 - m_1 \leq q_0$, then
\begin{align}
  & u_{k,m_2}^{(j)} \nonumber \\
= & (1-\lambda_1\eta)^{m_2-m_1}[u_{k,m_1}^{(j)} - \alpha_{m_2-m_1}\eta(z^{(j)}-\lambda_2)]
\end{align}
If $m_2 - m_1 > q_0$, according to the definition of $q_0$, we have
\begin{align}
    &(1-\lambda_1\eta)u_{k,m_1+q_0}^{(j)} - \eta z^{(j)} \nonumber \\
=   &(1-\lambda_1\eta)^{q_0+1}[u_{k,m_1}^{(j)} - \alpha_{q_0}\eta(z^{(j)}+\lambda_2)] - \eta z^{(j)}\nonumber \\
\in &(-\lambda_2\eta, \lambda_2\eta) \nonumber
\end{align}
Then we have $u_{k,m_1+q_0+1}^{(j)} = 0$. Since $0 - \eta z^{j} \in  (-\lambda_2\eta, \lambda_2\eta)$, then we have
\begin{align}
u_{k,m_2}^{(j)} = u_{k,m_2-1}^{(j)} = \cdots = u_{k,m_1+q_0+1}^{(j)} = 0
\end{align}

The next case is $z^{(j)} = -\lambda_2$. If $u_{k,m_1}^{(j)} \geq 0$, then
$(1 - \lambda_1\eta)u_{k,m_1}^{(j)} - \eta z^{(j)} \geq \lambda_2\eta$, which leads to $u_{k,m_1+1}^{(j)} = (1 - \lambda_1\eta)u_{k,m_1}^{(j)}$. By induction, we get that
\begin{align}
u_{k,m_2}^{(j)} = (1 - \lambda_1\eta)^{m_2 - m_1}u_{k,m_1}^{(j)}
\end{align}
If $u_{k,m_1}^{(j)} < 0$, define $q_0$ such that $\alpha_{q_0+1}\eta(z^{(j)}-\lambda_2) \leq u_{k,m_1}^{(j)} < \alpha_{q_0}\eta(z^{(j)}-\lambda_2)$. It is easy to verify that: for $q = 0,1,\ldots,\mathop{\min}\{q_0,m_2 - m_1\}$, we have
\begin{align}
u_{k,m_1+q}^{(j)} = (1-\lambda_1\eta)^q[u_{k,m_1}^{(j)} - \alpha_q\eta(z^{(j)}-\lambda_2)] \nonumber
\end{align}
So if $m_2 - m_1 \leq q_0$, then
\begin{align}
  & u_{k,m_2}^{(j)} \nonumber \\
= & (1-\lambda_1\eta)^{m_2-m_1}[u_{k,m_1}^{(j)} - \alpha_{m_2-m_1}\eta(z^{(j)}-\lambda_2)]
\end{align}
If $m_2 - m_1 > q_0$, according to the definition of $q_0$, we have
\begin{align}
    &(1-\lambda_1\eta)u_{k,m_1+q_0}^{(j)} - \eta z^{(j)} \nonumber \\
=   &(1-\lambda_1\eta)^{q_0+1}[u_{k,m_1}^{(j)} - \alpha_{q_0}\eta(z^{(j)}-\lambda_2)] - \eta z^{(j)}\nonumber \\
\in &(-\lambda_2\eta, \lambda_2\eta) \nonumber
\end{align}
Then we have $u_{k,m_1+q_0+1}^{(j)} = 0$. Since $0 - \eta z^{(j)} = \lambda_2\eta$, then we have
\begin{align}
u_{k,m_2}^{(j)} = u_{k,m_2-1}^{(j)} = \cdots = u_{k,m_1+q_0+1}^{(j)} = 0
\end{align}

The next case is $z^{(j)} = \lambda_2$. If If $u_{k,m_1}^{(j)} \leq 0$, then
$(1 - \lambda_1\eta)u_{k,m_1}^{(j)} - \eta z^{(j)} \leq -\lambda_2\eta$, which leads to $u_{k,m_1+1}^{(j)} = (1 - \lambda_1\eta)u_{k,m_1}^{(j)}$. By induction, we get that
\begin{align}
u_{k,m_2}^{(j)} = (1 - \lambda_1\eta)^{m_2 - m_1}u_{k,m_1}^{(j)}
\end{align}
If $u_{k,m_1}^{(j)} > 0$, define $q_0$ such that $\alpha_{q_0}\eta(z^{(j)}+\lambda_2) \leq u_{k,m_1}^{(j)} < \alpha_{q_0+1}\eta(z^{(j)}+\lambda_2)$. It is easy to verify that: for $q = 0,1,\ldots,\mathop{\min}\{q_0,m_2 - m_1\}$, we have
\begin{align}
u_{k,m_1+q}^{(j)} = (1-\lambda_1\eta)^q[u_{k,m_1}^{(j)} - \alpha_q\eta(z^{(j)}+\lambda_2)] \nonumber
\end{align}
So if $m_2 - m_1 \leq q_0$, then
\begin{align}
  & u_{k,m_2}^{(j)} \nonumber \\
= & (1-\lambda_1\eta)^{m_2-m_1}[u_{k,m_1}^{(j)} - \alpha_{m_2-m_1}\eta(z^{(j)}+\lambda_2)]
\end{align}
If $m_2 - m_1 > q_0$, according to the definition of $q_0$, we have
\begin{align}
    &(1-\lambda_1\eta)u_{k,m_1+q_0}^{(j)} - \eta z^{(j)} \nonumber \\
=   &(1-\lambda_1\eta)^{q_0+1}[u_{k,m_1}^{(j)} - \alpha_{q_0}\eta(z^{(j)}+\lambda_2)] - \eta z^{(j)}\nonumber \\
\in &(-\lambda_2\eta, \lambda_2\eta) \nonumber
\end{align}
Then we have $u_{k,m_1+q_0+1}^{(j)} = 0$. Since $0 - \eta z^{(j)} = -\lambda_2\eta$, then we have
\begin{align}
u_{k,m_2}^{(j)} = u_{k,m_2-1}^{(j)} = \cdots = u_{k,m_1+q_0+1}^{(j)} = 0
\end{align}

The next case is $z^{(j)} > \lambda_2$, $u_{k,m_1}^{(j)} \leq 0$, then $(1 - \lambda_1\eta)u_{k,m_1}^{(j)} - \eta z^{(j)} < -\lambda_2\eta$, which leads to $u_{k,m_1+1}^{(j)} = (1-\lambda_1\eta)u_{k,m_1}^{(j)} -  \eta(z^{(j)} - \lambda_2)$. By induction, we have
\begin{align}
u_{k,m_2}^{(j)} = & (1-\lambda_1\eta)^{m_2-m_1}u_{k,m_1}^{(j)} \nonumber \\
                  & -\eta(z^{(j)}-\lambda_2)\sum_{i=1}^{m_2-m_1}(1-\lambda_1\eta)^{i-1} \nonumber \\
                = & (1-\lambda_1\eta)^{m_2-m_1}u_{k,m_1}^{(j)} \nonumber \\
                  & -\eta(z^{(j)}-\lambda_2)\beta_{m_2-m_1}
\end{align}

The next case is $z^{(j)} > \lambda_2$, $u_{k,m_1}^{(j)} > 0$. Then define $q_0$ such that $\alpha_{q_0+1}\eta(z^{(j)}+\lambda_2)\leq u_{k,m_1}^{(j)} < \alpha_{q_0}\eta(z^{(j)}+\lambda_2)$. It is easy to verify that: for $q = 0,1,\ldots,\mathop{\min}\{q_0,m_2 - m_1\}$, we have
\begin{align}
u_{k,m_1+q}^{(j)} = (1-\lambda_1\eta)^q[u_{k,m_1}^{(j)} - \alpha_q\eta(z^{(j)}+\lambda_2)] \nonumber
\end{align}
So if $m_2 - m_1 \leq q_0$, then
\begin{align}
  & u_{k,m_2}^{(j)} \nonumber \\
= & (1-\lambda_1\eta)^{m_2 - m_1}q[u_{k,m_1}^{(j)} - \alpha_{m_2-m_1}\eta(z^{(j)}+\lambda_2)]
\end{align}
Otherwise, if $(1-\lambda_1\eta)u_{k,m_1+q_0}^{(j)} \leq \eta(z^{(j)} - \lambda_2)$, then by induction, we have
\begin{align}
  & u_{k,m_1+q_0+1}^{(j)} \nonumber \\
= & (1-\lambda_1\eta)u_{k,m_1+q_0}^{(j)} - \eta(z^{(j)} - \lambda_2) \nonumber \\
  & u_{k,m_1+q_0+2}^{(j)} \nonumber \\
= & (1-\lambda_1\eta)^2u_{k,m_1+q_0}^{(j)} - \eta(z^{(j)} - \lambda_2)(1+(1-\lambda_1\eta))\nonumber \\
\vdots \nonumber \\
  & u_{k,m_2}^{(j)} \nonumber \\
= & (1-\lambda_1\eta)^{m_2-m_1-q_0}u_{k,m_1+q_0}^{(j)} \nonumber \\
  &- \eta(z^{(j)} - \lambda_2)\beta_{m_2-m_1-q_0}
\end{align}
If $(1-\lambda_1\eta)u_{k,m_1+q_0}^{(j)} > \eta(z^{(j)} - \lambda_2)$, then
$(1-\lambda_1\eta)u_{k,m_1+q_0}^{(j)} - \eta z^{(j)} \in (-\lambda_2\eta, \lambda_2\eta)$, which leads to $u_{k,m_1+q_0+1}^{(j)} = 0$.

So if $m_2 - m_1 = q_0 + 1$, then
\begin{align}
u_{k,m_2}^{(j)} = 0.
\end{align}
If $m_2 - m_1 > q_0 + 1$, then by induction, we have
\begin{align}
& u_{k,m_1+q_0+2}^{(j)} = -\eta(z^{(j)}-\lambda_2) \nonumber \\
& \vdots \nonumber \\
& u_{k,m_2}^{(j)} = -\eta(z^{(j)}-\lambda_2)\beta_{m_2-m_1-q_0-1}
\end{align}

The next case is $z^{(j)} < -\lambda_2$, $u_{k,m_1}^{(j)} > 0$, then $(1 - \lambda_1\eta)u_{k,m_1}^{(j)} - \eta z^{(j)} > \lambda_2\eta$, which leads to $u_{k,m_1+1}^{(j)} = (1-\lambda_1\eta)u_{k,m_1}^{(j)} -  \eta(z^{(j)} + \lambda_2)$. By induction, we have
\begin{align}
u_{k,m_2}^{(j)} = & (1-\lambda_1\eta)^{m_2-m_1}u_{k,m_1}^{(j)} \nonumber \\
                  & -\eta(z^{(j)}+\lambda_2)\sum_{i=1}^{m_2-m_1}(1-\lambda_1\eta)^{i-1} \nonumber \\
                = & (1-\lambda_1\eta)^{m_2-m_1}u_{k,m_1}^{(j)} \nonumber \\
                  & -\eta(z^{(j)}+\lambda_2)\beta_{m_2-m_1}
\end{align}

The next case is $z^{(j)} < -\lambda_2$, $u_{k,m_1}^{(j)} \leq 0$. Then define $q_0$ such that $\alpha_{q_0+1}\eta(z^{(j)}-\lambda_2)\leq u_{k,m_1}^{(j)} < \alpha_{q_0}\eta(z^{(j)}-\lambda_2)$. It is easy to verify that: for $q = 0,1,\ldots,\mathop{\min}\{q_0,m_2 - m_1\}$, we have
\begin{align}
u_{k,m_1+q}^{(j)} = (1-\lambda_1\eta)^q[u_{k,m_1}^{(j)} - \alpha_q\eta(z^{(j)}-\lambda_2)] \nonumber
\end{align}
So if $m_2 - m_1 \leq q_0$, then
\begin{align}
  & u_{k,m_2}^{(j)} \nonumber \\
= & (1-\lambda_1\eta)^{m_2 - m_1}q[u_{k,m_1}^{(j)} - \alpha_{m_2-m_1}\eta(z^{(j)}-\lambda_2)]
\end{align}
Otherwise, if $(1-\lambda_1\eta)u_{k,m_1+q_0}^{(j)} > \eta(z^{(j)} +\lambda_2)$, then by induction, we have
\begin{align}
  & u_{k,m_1+q_0+1}^{(j)} \nonumber \\
= & (1-\lambda_1\eta)u_{k,m_1+q_0}^{(j)} - \eta(z^{(j)} + \lambda_2) \nonumber \\
  & u_{k,m_1+q_0+2}^{(j)} \nonumber \\
= & (1-\lambda_1\eta)^2u_{k,m_1+q_0}^{(j)} - \eta(z^{(j)} + \lambda_2)(1+(1-\lambda_1\eta))\nonumber \\
\vdots \nonumber \\
  & u_{k,m_2}^{(j)} \nonumber \\
= & (1-\lambda_1\eta)^{m_2-m_1-q_0}u_{k,m_1+q_0}^{(j)} \nonumber \\
  &- \eta(z^{(j)} + \lambda_2)\beta_{m_2-m_1-q_0}
\end{align}
If $(1-\lambda_1\eta)u_{k,m_1+q_0}^{(j)} < \eta(z^{(j)} + \lambda_2)$, then
$(1-\lambda_1\eta)u_{k,m_1+q_0}^{(j)} - \eta z^{(j)} \in (-\lambda_2\eta, \lambda_2\eta)$, which leads to $u_{k,m_1+q_0+1}^{(j)} = 0$.

So if $m_2 - m_1 = q_0 + 1$, then
\begin{align}
u_{k,m_2}^{(j)} = 0.
\end{align}
If $m_2 - m_1 > q_0 + 1$, then by induction, we have
\begin{align}
& u_{k,m_1+q_0+2}^{(j)} = -\eta(z^{(j)} + \lambda_2) \nonumber \\
& \vdots \nonumber \\
& u_{k,m_2}^{(j)} = -\eta(z^{(j)} + \lambda_2)\beta_{m_2-m_1-q_0-1}
\end{align}
\end{proof}

\end{document}